%% file: main.tex
\documentclass[preprint,12pt,authoryear]{elsarticle}
\input{preamble}

\AtBeginDocument{\hypersetup{citecolor=cyan,linkcolor=cyan}}
\journal{Neural Networks}
\usepackage{breakcites}

\long\def\black#1{{\color{black}#1}}

\begin{document}

\begin{frontmatter}	
\title{Neural Network Approximation: Three Hidden Layers Are Enough}

\author[nus]{Zuowei Shen}
\ead{matzuows@nus.edu.sg}

\author[pd]{Haizhao Yang}
\ead{haizhao@purdue.edu}

\author[nus]{Shijun Zhang}
\ead{zhangshijun@u.nus.edu}

\address[nus]{Department of Mathematics,  National University of Singapore}

\address[pd]{Department of Mathematics, Purdue University}

\begin{abstract}
	
A three-hidden-layer neural network with super approximation power is introduced. This network is built with the floor function ($\lfloor x\rfloor$), the exponential function ($2^x$), the step function ($\one_{x\geq 0}$), or their compositions as the activation function in each neuron and hence we call such networks as Floor-Exponential-Step (FLES) networks. 
For any width hyper-parameter $N\in\mathbb{N}^+$, it is shown that FLES networks with  width $\max\{d,N\}$ and three hidden layers can uniformly approximate a H\"older continuous function $f$ on $[0,1]^d$ with an exponential approximation rate $3\lambda (2\sqrt{d})^{\alpha} 2^{-\alpha N}$, where $\alpha \in(0,1]$ and $\lambda>0$ are the H\"older order and constant, respectively. More generally  for an arbitrary continuous function $f$ on $[0,1]^d$ with a modulus of continuity $\omega_f(\cdot)$, the constructive approximation rate is $2\omega_f(2\sqrt{d}){2^{-N}}+\omega_f(2\sqrt{d}\,2^{-N})$. Moreover, we extend such a result to general bounded continuous functions on a bounded set $E\subseteq\mathbb{R}^d$. As a consequence, this new class of networks overcomes the curse of dimensionality in approximation power when the variation of $\omega_f(r)$ as $r\rightarrow 0$ is moderate (e.g., $\omega_f(r)\lesssim r^\alpha$ for H\"older continuous functions), since the major term to be concerned in our approximation rate is essentially $\sqrt{d}$ times a function of $N$ independent of $d$ within the modulus of continuity. Finally, we extend our analysis to derive similar approximation results in the $L^p$-norm for $p\in[1,\infty)$  via replacing   Floor-Exponential-Step activation functions by  continuous activation functions.

\end{abstract}


\end{frontmatter}


 
\section{Introduction}

This paper studies the approximation power of neural networks and shows that three hidden layers are enough for neural networks to achieve super approximation capacity. In particular, leveraging the power of advanced yet simple activation functions, we will introduce new theories and network architectures with only three hidden layers achieving exponential convergence and avoiding the curse of dimensionality simultaneously for (H\"older) continuous functions with an explicit approximation bound. The theories established in this paper would provide new insights to explain why deeper neural networks are better than one-hidden-layer neural networks for large-scale and high-dimensional problems. The approximation theories here are constructive (i.e., with explicit formulas to specify network parameters) and quantitative (i.e., results valid for essentially arbitrary width and/or depth without lower bound constraints) 
with explicit error bounds working for three-hidden-layer networks with arbitrary width. 

Constructive approximation with quantitative results and explicit error bounds would provide important guides for deciding the network sizes in deep learning. For example, the (nearly) optimal approximation rates of deep ReLU networks \black{with width $\calO(N)$ and depth $\calO(L)$} for a Lipschitz continuous function and a $C^s$ function $f$ on $[0,1]^d$ are $\calO(\sqrt{d}N^{-2/d}L^{-2/d})$ and $\calO(\|f\|_{C^s}(\tfrac{N}{\ln N})^{-2s/d}(\tfrac{L}{\ln L})^{-2s/d})$  \cite{Shen2,Shen3}, respectively. For results in terms of the number of nonzero parameters, the reader is referred to \cite{yarotsky2017,Johannes,PETERSEN2018296,yarotsky18a,guhring2019error,yarotsky2019} and the reference therein. Obviously, the curse of dimensionality exists in ReLU networks for these generic functions and, therefore, ReLU networks would need to be exponentially large in $d$ to maintain a reasonably good approximation accuracy. The curse could be lessened when target function spaces are \black{smaller}. To name a few, \cite{poggio2017,barron2018approximation,Weinan2019,bandlimit,Wenjing,Hutzenthaler2018} and the reference therein for ReLU networks. The limitation of ReLU networks motivated the work in \cite{Shen4} to introduce Floor-ReLU networks built with either a Floor ($\lfloor x\rfloor$) or ReLU ($\max\{0,x\}$) activation function in each neuron. It was shown by construction in \cite{Shen4} that Floor-ReLU networks with width $\max\{d,\, 5N+13\}$ and depth $64dL+3$ can uniformly approximate a H{\"o}lder continuous function $f$ on $[0,1]^d$ with a root-exponential approximation rate $3\lambda d^{\alpha/2}N^{-\alpha\sqrt{L}}$ without the curse of dimensionality, where $\alpha\in (0,1]$ and $\lambda>0$ are the H\"older order and constant, respectively.

The most important message of \cite{Shen4} (and probably also of \cite{yarotsky2019}) is that the combination of simple activation functions can create super approximation power. In the Floor-ReLU networks mentioned above, the power of depth is fully reflected in the approximation rate $3\lambda d^{\alpha/2}N^{-\alpha\sqrt{L}}$ that is root-exponential in depth. However, the power of width is much weaker and the approximation rate is polynomial in width if depth is fixed. This seems to be inconsistent with recent development of network optimization theory \cite{Arthur18,du2018gradient,MeiE7665,NIPS2018_8049,Chen1,Yiping20,Luo2020TwoLayerNN}, where larger width instead of depth can ease the challenge of highly noncovex optimization. The mystery of the power of width and depth remains and it motivates us to demonstrate that width can also enable super approximation power when armed with appropriate activation functions. 

In particular, we explore the floor function, the exponential function ($2^x$), the step function ($\one_{x\geq 0}$), or their compositions as activation functions to build fully-connected feed-forward neural networks. These networks are called Floor-Exponential-Step (FLES) networks. As we shall prove by construction, Theorem~\ref{thm:main} below shows that FLES networks with width $\max\{d,N\}$ and three hidden layers can uniformly approximate a continuous function $f$ on $[0,1]^d$ with an exponential approximation rate $2\omega_f(2\sqrt{d})2^{-N}+\omega_{f}(2\sqrt{d}\,2^{-N})$, where $\omega_f(\cdot)$ is the modulus of continuity defined as 
\begin{equation*}
	\omega_f(r)\coloneqq \sup\big\{|f(\bmx)-f(\bmy)|: \|\bmx-\bmy\|_2\le r,\ \bmx,\bmy\in [0,1]^d\big\},\quad\tn{for any $r\ge0$.}
\end{equation*}
In particular, there are three kinds of activation functions denoted as $\sigma_1$, $\sigma_2$, and $\sigma_3$ in FLES networks (see Figure~\ref{fig:threeHL} for an illustration):
\begin{equation*}
	\sigma_1(x)\coloneqq\lfloor x\rfloor,\quad \sigma_2(x)\coloneqq 2^x,\quad \tn{and}\quad \sigma_3(x)\coloneqq\calT(x-\lfloor x\rfloor-\tfrac12),\quad \tn{for any $x\in \R$,}
\end{equation*}
where \[\calT(x)\coloneqq \one_{x\geq 0} =
\left\{\genfrac{}{}{0pt}{0}{
	1,\   x\ge 0,}
	{0,\  x<0,}
	\right.
	\quad \tn{ for any $x\in \R.$}\]

\begin{theorem}\label{thm:main}
	Given an arbitrary continuous function $f$ defined on $[0,1]^d$, for any $N\in \N^+$, there exist $a_1,a_2,\cdots,a_ N\in[0,\tfrac12)$ such that
	\begin{equation*}
		|\phi(\bmx)-f(\bmx)|\le 2\omega_f(2\sqrt{d})2^{-N}+\omega_{f}(2\sqrt{d}\,2^{-N}),
	\end{equation*}
	for any $\bmx=(x_1,x_2,\cdots,x_d)\in [0,1]^d$, where $\phi$ is defined by a formula in $a_1,a_2,\cdots,a_N$ as follows.
	\begin{equation*}
			\mathResize[0.92]{
			\phi(\bmx)=2\omega_f(2\sqrt{d})\sum_{j=1}^ N 2^{-j} \sigma_3\bigg(\, a_j\cdot\sigma_2\Big(\, 1+\sum_{i=1}^{d} 2^{(i-1)N}\sigma_1(2^{N-1} x_i)\, \Big)\, \bigg)+f(\bmzero)-\omega_f(2\sqrt{d}). }
	\end{equation*}
\end{theorem}

\black{We remark that $\phi$ in Theorem~\ref{thm:main} is essentially determined by $N$ parameters $a_1,a_2,\cdots,a_N$, which can be trained by a $(\sigma_1,\sigma_2,\sigma_3)$-activated network with width $\max\{d,N\}$, three hidden layers, and $2(d+N+1)$ nonzero parameters. See Figure~\ref{fig:threeHL} for an illustration.}

\begin{figure}[!htp]
	\centering
	\includegraphics[width=0.9888\textwidth]{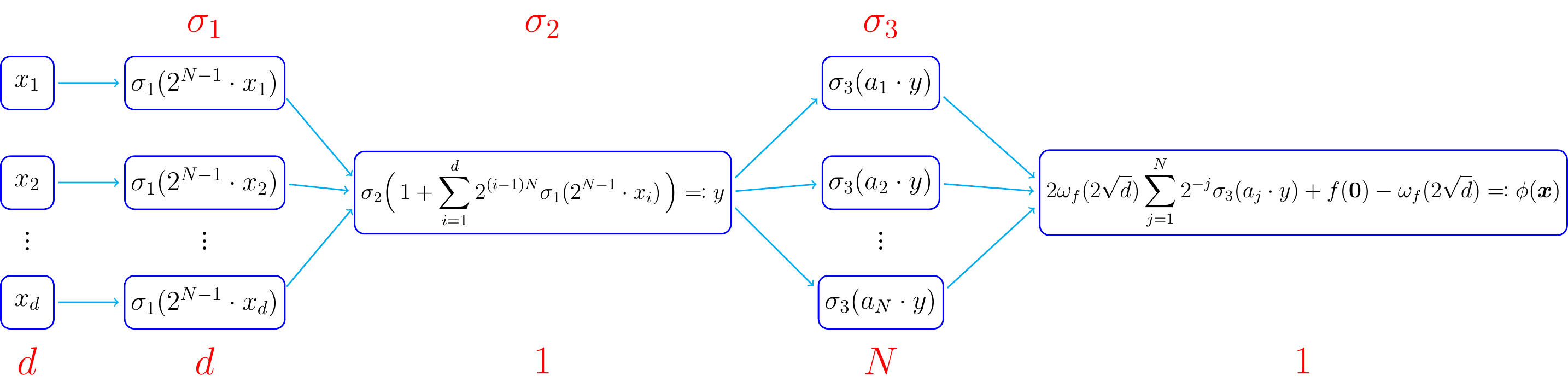}
	\caption{An illustration of the desired three-hidden-layer network in Theorem~\ref{thm:main} for any $\bmx=(x_1,x_2,\cdots,x_d)\in\R$. Each of the red functions ``$\textcolor{red}{\sigma_1}$'', ``$\textcolor{red}{\sigma_2}$'', and ``$\textcolor{red}{\sigma_3}$'' above the network is the activation function of the corresponding hidden layer.  The number of neurons in \black{each} hidden layer is indicated by the red number below it.}
	\label{fig:threeHL}
\end{figure}


 The rate in $\omega_f(2\sqrt{d}\,2^{-N})$ implicitly depends on $N$ through the modulus of continuity of $f$, while the rate in $2\omega_f(2\sqrt{d}){2^{-N}}$ is explicit in $N$. Simplifying the implicit approximation rate to make it explicitly depend on $N$ is challenging in general. However, if $f$ is a H{\"o}lder continuous function on $[0,1]^d$ of order $\alpha\in(0,1]$ with a H\"older constant $\lambda>0$, i.e., $f(\bmx)$ satisfying
\begin{equation}\label{eqn:Holder}
	|f(\bmx)-f(\bmy)|\leq \lambda \|\bmx-\bmy\|_2^\alpha,\quad \tn{for any $\bmx,\bmy\in[0,1]^d$,}
\end{equation}
then $\omega_f(r)\le \lambda r^\alpha$ for any $r\ge 0$. Therefore, in the case of H{\"o}lder continuous functions, the approximation rate is simplified to $3\lambda (2\sqrt{d})^{\alpha}2^{-\alpha N}$ as shown in the following corollary. In the special case of Lipschitz continuous functions with a Lipschitz constant $\lambda>0$, the approximation rate is simplified to $6\lambda\sqrt{d}\,2^{-N}$.

\begin{corollary}
	\label{coro:main}
	Given any H{\"o}lder continuous function $f$ on $[0,1]^d$ of order $\alpha\in(0,1]$ with a H\"older constant $\lambda>0$, for any $N\in \N^+$, there exists
	$a_1,a_2,\cdots,a_N$
such that
	\begin{equation*}
		|\phi(\bmx)-f(\bmx)|\le 3\lambda (2\sqrt{d})^{\alpha}{2^{-\alpha N}},\quad \tn{for any $\bmx=(x_1,x_2,\cdots,x_d)\in [0,1]^d$,}
	\end{equation*}
	where $\phi$ is defined by a formula in $a_1,a_2,\cdots,a_N$ as follows.
	\begin{equation*}
			\mathResize[0.92]{
			\phi(\bmx)=2\omega_f(2\sqrt{d})\sum_{j=1}^ N 2^{-j} \sigma_3\bigg(\, a_j\cdot \sigma_2\Big(\, 1+\sum_{i=1}^{d} 2^{(i-1)N}\sigma_1(2^{N-1} x_i)\, \Big)\, \bigg)+f(\bmzero)-\omega_f(2\sqrt{d}). 
			}
	\end{equation*}
\end{corollary}

First, Theorem~\ref{thm:main} and Corollary~\ref{coro:main} show that the approximation capacity of three-hidden-layer neural networks with simple activation functions for continuous functions can be exponentially improved by increasing the network width, and the approximation error can be explicitly characterized in terms of the width $\calO(N)$. Second, this new class of networks overcomes the curse of dimensionality in the approximation power when the modulus of continuity is moderate, since the approximation order is essentially $\sqrt{d}$ times a function of $N$ independent of $d$ within the modulus of continuity. Therefore, three hidden layers are enough for neural networks to achieve exponential convergence and avoid the curse of dimensionality for generic functions. The width is also powerful in network approximation.

The rest of this paper is organized as follows. \black{In Section~\ref{sec:dis}, we discuss the application scope of our theory, study the connection between the approximation error and the Vapnik-Chervonenkis (VC) dimension, establish Corollary~\ref{cor:subsetE} to extend our analysis to general bounded continuous functions on a bounded set, and compare related works in the literature. We will prove Theorem~\ref{thm:main} and Corollary~\ref{cor:subsetE} in Section~\ref{sec:approxContFunc}. 
In Section~\ref{sec:more}, we explore alternative continuous activation functions other than $\sigma_1$, $\sigma_2$, and $\sigma_3$ for super approximation power.
}
Finally, we conclude this paper in Section~\ref{sec:conclusion}.

\section{Discussion}
\label{sec:dis}
In this section, we will further interpret our results and discuss related research in the field of neural network approximation. 

\subsection{Application scope of our theory in machine learning}

Let $\phi(\bm{x};\bm{\theta})$ denote a function computed by  a (fully-connected) network with $\bm{\theta}$ as the set of parameters. 
Given a target function $f$, consider the expected error/risk of $\phi(\bm{x};\bm{\theta})$
\begin{equation*}
	R_{\mathcal{D}}(\bm{\theta})\coloneqq \mathbb{E}_{\bm{x}\sim U(\mathcal{X})} \left[\ell( \phi(\bm{x};\bm{\theta}),f(\bm{x}))\right]
\end{equation*}
with a loss function typically taken as $\ell(y,y')=\tfrac{1}{2}|y-y'|^2$,
where $U(\mathcal{X})$ is an unknown data {distribution}  over $\mathcal{X}$.
For example, when $\ell(y,y')=\tfrac{1}{2}|y-y'|^2$ and $U$ is a uniform distribution over $\mathcal{X}=[0,1]^d$,
\begin{equation*}
	R_{\mathcal{D}}(\bm{\theta})=\int_{[0,1]^d} \tfrac12 |\phi(\bmx;\bm{\theta})-f(\bmx)|^2 d\bmx.
\end{equation*}
The expected risk minimizer $\bm{\theta}_{\mathcal{D}}$ is defined as 
\begin{equation*}
	\bm{\theta}_{\mathcal{D}}\coloneqq \argmin_{\bm{\theta}} R_{\mathcal{D}}(\bm{\theta}).
\end{equation*}
It 
is unachievable in practice since $f$ and $U(\mathcal{X})$ are not available.  Instead, we only have samples of $f$.

Given {samples} $\{( \bm{x}_i,f(\bm{x}_i))\}_{i=1}^n$,   the empirical risk is  defined as
\begin{equation*}
	R_{\mathcal{S}}(\bm{\theta}):=
	\frac{1}{n}\sum_{i=1}^n \ell\big( \phi(\bm{x}_i;\bm{\theta}),f(\bm{x}_i)\big).
\end{equation*}
And we usually use it to approximate/model  the expected risk $R_\calD(\bmtheta)$.
The goal of supervised learning
is to identify  the empirical risk minimizer
\begin{equation}\label{eqn:emloss}
	\bm{\theta}_{\mathcal{S}}=\argmin_{\bm{\theta}}R_{\mathcal{S}}(\bm{\theta}),
\end{equation}
to obtain $\phi(\bm{x};\bm{\theta}_{\mathcal{S}})\approx f(\bm{x})$.
When a numerical optimization method is applied to solve  \eqref{eqn:emloss}, it may result in a numerical solution (denoted as $\bm{\theta}_{\mathcal{N}}$) that is not a global minimizer. Hence, the actually learned function generated by a neural network  is $\phi(\bm{x};\bm{\theta}_{\mathcal{N}})$. The discrepancy between the target function $f$ and the actually learned function $\phi(\bmx;\bmtheta_\calN)$ is measured by an inference error 
\[R_{\mathcal{D}}(\bm{\theta}_{\mathcal{N}})=\mathbb{E}_{\bm{x}\sim U(\mathcal{X})} \left[\ell( {\phi(\bm{x};\bm{\theta}_\calN)},{f(\bm{x})})\right] \mathop{=}^{e.g.} \int_{[0,1]^d} \tfrac12 |{\phi(\bmx;\bm{\theta}_\calN)}-{f(\bmx)}|^2 d\bmx,\]
where the second equality holds when $\ell(y,y')=\tfrac{1}{2}|y-y'|^2$ and $U$ is a uniform distribution over $\mathcal{X}=[0,1]^d$.

Since $R_{\mathcal{D}}(\bm{\theta}_{\mathcal{N}})$ is the expected inference error over all possible data samples, it can quantify how good the learned function $\phi(\bm{x};\bm{\theta}_{\mathcal{N}})$ is. Note that
\setMathResizeRate{0.71} 
\begin{align}\label{eqn:gen}	
	\mathResize{   
		R_{\mathcal{D}}(\bm{\theta}_{\mathcal{N}})    
	} 
	& \mathResize{   
		 =\underbrace{[R_{\mathcal{D}}(\bm{\theta}_{\mathcal{N}})-R_{\mathcal{S}}(\bm{\theta}_{\mathcal{N}})]}_{\tn{GE}}
		+\underbrace{[R_{\mathcal{S}}(\bm{\theta}_{\mathcal{N}})-R_{\mathcal{S}}(\bm{\theta}_{\mathcal{S}})]}_{\tn{OE}}   
		+\underbrace{[R_{\mathcal{S}}(\bm{\theta}_{\mathcal{S}})-R_{\mathcal{S}}(\bm{\theta}_{\mathcal{D}})]}_{\tn{$\le 0$ by Eq. \eqref{eqn:emloss}}} +\underbrace{[R_{\mathcal{S}}(\bm{\theta}_{\mathcal{D}})-R_{\mathcal{D}}(\bm{\theta}_{\mathcal{D}})]}_{\tn{GE}}
		+\underbrace{R_{\mathcal{D}}(\bm{\theta}_{\mathcal{D}})}_{\tn{AE}}
	}
	\nonumber \\
	&\mathResize{  
		\le \underbrace{R_{\mathcal{D}}(\bm{\theta}_{\mathcal{D}})}_{\tn{\color{blue}Approximation error (AE)}} \ +\  \underbrace{[R_{\mathcal{S}}(\bm{\theta}_{\mathcal{N}})-R_{\mathcal{S}}(\bm{\theta}_{\mathcal{S}})]}_{\tn{\color{blue}Optimization error (OE)}}\  + \  \underbrace{[R_{\mathcal{D}}(\bm{\theta}_{\mathcal{N}})-R_{\mathcal{S}}(\bm{\theta}_{\mathcal{N}})]
			+[R_{\mathcal{S}}(\bm{\theta}_{\mathcal{D}})-R_{\mathcal{D}}(\bm{\theta}_{\mathcal{D}})]}_{\tn{\color{blue}Generalization error (GE)}},  
		}
\end{align}
where the inequality comes from the fact that $[R_{\mathcal{S}}(\bm{\theta}_{\mathcal{S}})-R_{\mathcal{S}}(\bm{\theta}_{\mathcal{D}})]\leq 0$ since $\bm{\theta}_{\mathcal{S}}$ is a global minimizer of $R_{\mathcal{S}}(\bm{\theta})$. Constructive approximation provides an upper bound of $R_{\mathcal{D}}(\bm{\theta}_{\mathcal{D}})$ in terms of the network size, e.g., in terms of the network width and depth, or in terms of the number of parameters. The second term of Equation~\eqref{eqn:gen} is bounded by the optimization error of the numerical algorithm applied to solve the empirical loss minimization problem in Equation~\eqref{eqn:emloss}. Note that one only needs to make  $R_{\mathcal{S}}(\bm{\theta}_{\mathcal{N}})-R_{\mathcal{S}}(\bm{\theta}_{\mathcal{S}})$ small, but not $\bmtheta_\calN-\bmtheta_\calS$.
The study of the bounds for the third and fourth terms is referred to as the generalization error analysis of neural networks. See Figure~\ref{fig:AEOEGE} for the intuitions of these three errors.

One of the key targets in  the area of deep learning is to develop algorithms to reduce  $R_\calD{(\bmtheta_\calN)}$.
The constructive approximation established in this paper and in the literature provides upper bounds of the approximation  error $R_{\mathcal{D}}(\bm{\theta}_{\mathcal{D}})$ for several function spaces, which is crucial to estimate an upper bound of $R_\calD{(\bmtheta_\calN)}$.   Instead of deriving an approximator to attain  the approximation error bound,  deep learning algorithms aim  to identify a solution $\phi(\bm{x};\bm{\theta}_{\mathcal{N}})$ reducing the generalization and optimization errors in Equation~\eqref{eqn:gen}.  Solutions minimizing both generalization and optimization errors will lead to a good solution only if we also have a good upper bound estimate of $R_{\mathcal{D}}(\bm{\theta}_{\mathcal{D}})$ as shown in Equation~\eqref{eqn:gen}.   Independent of whether our analysis here  leads to a good approximator, which is an interesting topic to pursue,  the  theory here does provide a key ingredient in the error analysis of deep learning algorithms.

\begin{figure}
	\centering
	\includegraphics[width=0.86\textwidth]{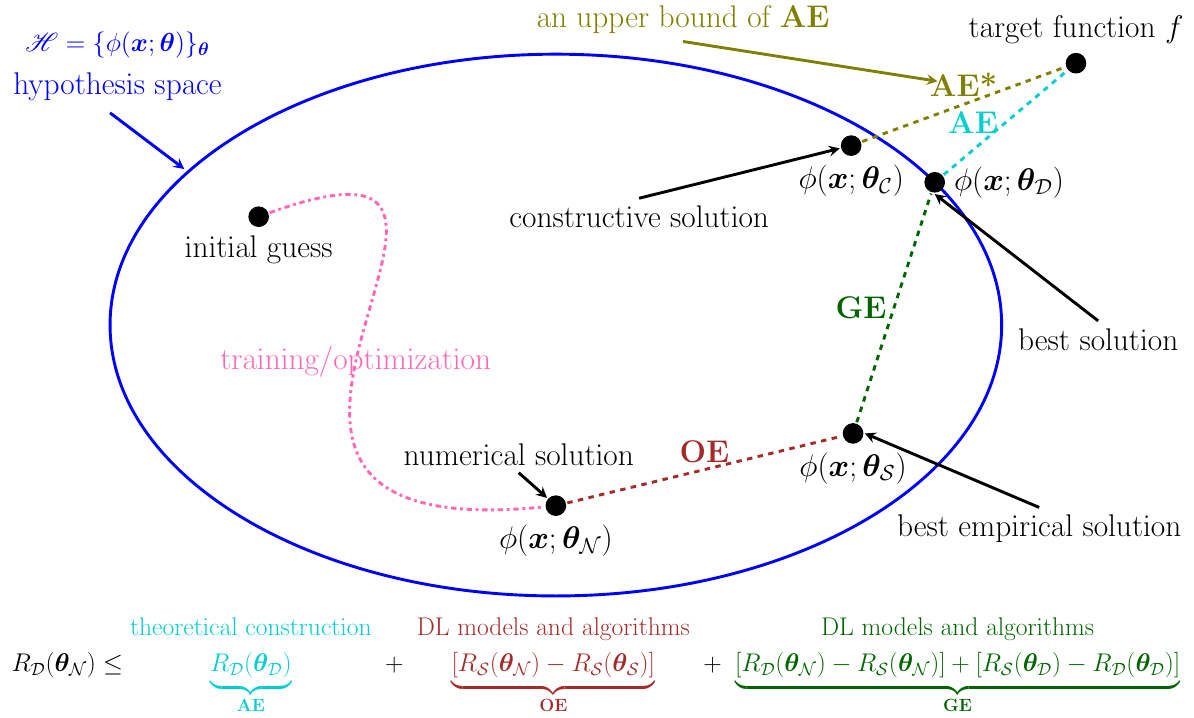}	
	\caption{The intuitions of the approximation error (AE), the optimization error (OE), and the generalization error (GE). DL is short of deep learning. One needs to control AE, OE, and GE in order to bound the discrepancy between {the target function $f$} and {the numerical solution $\phi(\bmx;\bmtheta_\calN)$ (what we can get in practice)}, measured by $\displaystyle R_{\mathcal{D}}(\bm{\theta}_{\mathcal{N}})=\mathbb{E}_{\bm{x}\sim U(\mathcal{X})} \left[\ell( {\phi(\bm{x};\bm{\theta}_\calN)},{f(\bm{x})})\right]$.} 
	\label{fig:AEOEGE}
\end{figure}

Theorem~\ref{thm:main} and Corollary~\ref{coro:main} provide an upper bound of $R_{\mathcal{D}}(\bm{\theta}_{\mathcal{D}})$. This bound only depends on the given budget of neurons {and layers} of FLES networks. Hence, this bound is independent of the empirical loss minimization in Equation~\eqref{eqn:emloss} and the optimization algorithm used to compute the numerical solution of Equation~\eqref{eqn:emloss}. In other words, Theorem~\ref{thm:main} and Corollary~\ref{coro:main} quantify the approximation power of FLES networks with a given size. Designing efficient optimization algorithms and analyzing the generalization bounds for FLES networks are two other separate future directions.

\black{
\subsection{Connection between approximation error and VC-dimension}\label{sec:approx:vcdim}

The approximation error and the Vapnik-Chervonenkis (VC) dimension are two important measures of the capacity (complexity) of a set of functions.
In this section, we discuss the connection between them.

Let us first present the definitions of VC-dimension and related concepts. Assume $H$ is a class of functions mapping from a general domain $\mathcal{X}$ to $\{0,1\}$. 
We say $H$ shatters a set of points $\{\bmx_1,\bmx_2,\cdots,\bmx_m\}\subseteq \mathcal{X}$ if
\begin{equation*}
\Big| \Big\{\big[h(\bmx_1),h(\bmx_2),\cdots,h(\bmx_m)\big]^T\in \{0,1\}^m: h\in H\Big\}\Big|=2^m,
\end{equation*}
where $|\cdot|$ means the size of a set. The above equation means, given any $\theta_i\in \{0,1\}$ for $i=1,2,\cdots,m$, there exists $h\in H$ such that
$h(\bmx_i)=\theta_i$ for all $i$. For a general function set $\scrF$ with its elements mapping from $\mathcal{X}$ to $\R$, we say $\scrF$ shatters $\{\bmx_1,\bmx_2,\cdots,\bmx_m\}\subseteq \mathcal{X}$ if $\calT\circ \scrF$ does, where \begin{equation*}
		\calT(t)\coloneqq \genfrac{\{}{.}{0pt}{0}{1,\ t\ge 0,}{0, \ t< 0\phantom{,}} \quad \tn{and}\quad \calT\circ \scrF\coloneqq \{\calT\circ f: f\in \scrF\}.
\end{equation*}

For any $m\in \N^+$, the growth function of $H$ is defined as 
\begin{equation*}
\Pi_H(m)\coloneqq \max_{\bmx_1,\bmx_2,\cdots,\bmx_m\in \mathcal{X}} \Big| \Big\{\big[h(\bmx_1),h(\bmx_2),\cdots,h(\bmx_m)\big]^T\in \{0,1\}^m: h\in H\Big\}\Big|.
\end{equation*}
	
\begin{definition}[VC-dimension]
Assume $H$ is a class of functions from $\mathcal{X}$ to $\{0,1\}$.
The VC-dimension of $H$, denoted by  $\vcd(H)$, is the size of the largest shattered set, namely, 
\begin{equation*}
    \vcd(H)\coloneqq \sup \big\{ m\in\N^+ : \Pi_H(m)=2^m\big\}
\end{equation*}
in the case that $\{m\in\N^+:\Pi_H(m)=2^m \}$ is not empty. If $\{m\in\N^+:\Pi_H(m)=2^m \}=\emptyset$, we may define $\vcd(H)=0$. 
	
		Let $\scrF$ be a class of functions from $\mathcal{X}$ to $\R$. The VC-dimension of $\scrF$, denoted by $\vcd(\scrF)$, is defined by $\vcd(\scrF)\coloneqq\vcd(\calT\circ\scrF)$,\footnote{One may also define $\vcd(\scrF)\coloneqq\vcd(\widehat{\calT}\circ\scrF)$, where $\widehat{\calT}(t)\coloneqq \genfrac{\{}{.}{0pt}{1}{1,\ t> 0,}{0, \ t\le 0\phantom{,}}$.}
		where
	\begin{equation*}
	    \calT(t)\coloneqq \genfrac{\{}{.}{0pt}{0}{1,\ t\ge 0,}{0, \ t< 0\phantom{,}} \quad \tn{and}\quad \calT\circ \scrF\coloneqq \{\calT\circ f: f\in \scrF\}.
	\end{equation*}
	In particular, the expression ``the VC-dimension of a network (architecture)'' means the VC-dimension of the function set that consists of all functions computed by this network (architecture).
\end{definition}

As shown in \cite{yarotsky18a,yarotsky2017,Shen1,Shen2,Shen3,Shen4,shijun:thesis,2021arXiv210300502S}, VC-dimension essentially determines the lower bound of the approximation errors of networks.
For simplicity, we use $\holder([0,1]^d,\alpha,\lambda)$ as an example, where $\holder([0,1]^d,\alpha,\lambda)$ denotes the space of H\"older continuous functions of order $\alpha\in (0,1]$ and a H\"older constant $\lambda>0$. 
Without loss of generality, we assume $\lambda=1$.
Theorem~\ref{thm:linkVcdRate} below  shows that the best possible approximation error of functions in $\holder([0,1]^d,\alpha,1)$ approximated by functions in $\scrF$ is bounded by a formula characterized by $\vcd(\scrF)$.

\begin{theorem}[Theorem 2.4 of \cite{2021arXiv210300502S} or Theorem $4.17$ of \cite{shijun:thesis} ]
	\label{thm:linkVcdRate}
	Assume $\scrF$ is a function set  with all elements defined on $[0,1]^d$. Given any $\varepsilon>0$, suppose $\vcd(\scrF)\ge 1$ and
	\begin{equation*}
		\inf_{\phi\in \scrF}\|\phi-f\|_{L^\infty([0,1]^d)}\le \varepsilon,\quad \tn{for any $f\in \holder([0,1]^d,\alpha,1)$.}
	\end{equation*}
	Then $\vcd(\scrF)\ge (9\varepsilon)^{-d/\alpha}$.
\end{theorem}

This theorem investigates the connection between VC-dimension of $\scrF$ and the approximation errors of functions in $\holder([0,1]^d,\alpha,1)$ approximated by elements of $\scrF$. 
Denote the best approximation error of functions in $\holder([0,1]^d,\alpha,1)$ approximated by the elements of $\scrF$  as 
\begin{equation*}
	\calE_{\alpha,d}(\scrF) \coloneqq \sup_{f\in \holder([0,1]^d,\alpha,1)} \Big(\inf_{\phi\in \scrF } \|\phi-f\|_{L^\infty([0,1]^d)}\Big).
\end{equation*}
Then, Theorem~\ref{thm:linkVcdRate} implies that 
\begin{equation*}
	\vcd(\scrF)^{-\alpha/d}\big/9 \   \le \    \calE_{\alpha,d}(\scrF),
\end{equation*} 
which means that the best possible approximation error is controlled by $\vcd(\scrF)^{-\alpha/d}/9$.
A typical application of this theorem is to prove the optimality of approximation errors when using ReLU networks to approximate functions in $\holder([0,1]^d,\alpha,1)$. 
It is shown in \cite{pmlr-v65-harvey17a} that 
the VC-dimension of $\scrF_{N,L}$ is bounded by 
\[\vcd(\scrF_{N,L})
\le \calO\big(N^2L\cdot L \cdot\ln(N^2L)\big)
\le\calO\big(N^2L^2\ln(NL)\big),\]
 where  $\scrF_{N,L}$ is the space consisting of all functions implemented by ReLU networks with width $N$ and depth $L$.
 It is shown in Section $4.4.1$ of \cite{shijun:thesis} that
\begin{equation*}
	C_1(\alpha,d)\cdot \Big(N^2L^2 {\ln (NL)}\Big)^{-\alpha/d}
	\ \le\ 
	\calE_{\alpha,d}(\scrF_{N,L}) 
	\ \le\  
	C_2(\alpha,d)\cdot\Big({N^2L^2}\Big)^{-\alpha/d},
\end{equation*}
where $C_1(\alpha,d)$ and $C_2(\alpha,d)$ are two positive constants determined by $\alpha$ and $d$.

Finally, we would like to point out that a {large} VC-dimension of the hypothesis space $\scrF$ is a \textbf{necessary} condition of a {good} approximation error, but cannot guarantee a good approximation error, which also relies on other properties of the hypothesis space $\scrF$. For example, it is easy to check by Proposition~\ref{prop:bitsExtract:cos} that 
\begin{equation*}
	\vcd\Big(\big\{\phi: \phi(x)=\cos(ax),\  a\in\R \big\}\Big)=\infty.
\end{equation*}
However, $\big\{\phi: \phi(x)=\cos(ax),\  a\in\R \big\}$ cannot achieve a good approximation error when approximating H\"older continuous functions. Designing a hypothesis space with a large VC-dimension is the first step for a good approximation toll, but to realize the desired approximation power requires refined design of the hypothesise space, which is also the philosophy we followed in this paper. Our initial goal is to design a network architecture with a fixed depth (e.g., three hidden layers) to generate  a hypothesis space with a sufficiently large VC-dimension ($\infty$).  As we shall see later,  Proposition~\ref{prop:bitsExtract} implies that the VC-dimension of FLES networks is infinity,  which is a necessary condition for our FLES networks to attain super approximation power.
}

\subsection{Further interpretation of our theory}

In the interpretation of our theory, three more aspects are important to discuss. The first one is whether it is possible to extend our theory to functions on a more general domain, e.g, $E\subseteq [-R, R]^d$ for any $R>0$, because $R>1$ may cause an implicit curse of dimensionality in some existing theory. The second one is how bad the modulus of continuity would be since it is related to a high-dimensional function $f$ that may lead to an implicit curse of dimensionality in our approximation rate. \black{The last one is the discussion of overcoming the  zero derivative  in training FLES networks.}

\black{First, we can generalize Theorem~\ref{thm:main} to the function space $C(E)$ with $E\subseteq [-R,R]^d$ for any $R>0$ in the following corollary with the modulus of continuity $\omega_f^{E}(\cdot)$ defined as follows. For an arbitrary set $E\subseteq\R^d$,
$\omega_f^E(r)$ is defined via 
\[\omega_f^E(r)\coloneqq  \sup\big\{|f(\bmx)-f(\bmy)|: \|\bmx-\bmy\|_2\le r,\ \bmx,\bmy\in E\big\},\quad \tn{ for any $r\ge0$.}\] 
As defined earlier, in the case $E=[0,1]^d$, $\omega_f^{E}(r)$ is abbreviated to $\omega_f(r)$. The proof of this corollary will be presented in Section~\ref{sec:proof:main}.
	
\begin{corollary}\label{cor:subsetE}
	Given an arbitrary bounded continuous function $f$ on $E\subseteq [-R,R]^d$ where $R$ is an arbitrary positive real number, for  any $N\in \N^+$, there exist $a_1,a_2,\cdots,a_ N\in[0,\tfrac12)$ such that
	\begin{equation*}
		|\phi(\bmx)-f(\bmx)|\le 2\omega_f^{E}(3R\sqrt{d})2^{-N}+\omega_{f}^{E}(3R\sqrt{d}\,2^{-N}),
	\end{equation*}
	for any $\bmx=(x_1,x_2,\cdots,x_d)\in E$, where  $\phi$ is defined by a formula in $a_1,a_2,\cdots,a_N$ as follows.
	\begin{equation*}
	\phi(\bmx)=2\omega_f^{E}(3R\sqrt{d})\sum_{j=1}^ N 2^{-j} \sigma_3\bigg(\, a_j\cdot \sigma_2\Big(\, 1+\sum_{i=1}^{d} 2^{(i-1)N}\sigma_1(2^N\tfrac{x_i+R}{3R})\, \Big)\, \bigg)+C_f,
	\end{equation*}
	where $C_f$ is a constant determined by $f$.
\end{corollary}
	
	Hence, the volume of the function domain $E\subseteq [-R,R]^d$ only has a mild influence on the approximation rate of our FLES networks. FLES networks can still avoid the curse of dimensionality and achieve exponential convergence for continuous functions on $E\subseteq [-R,R]^d$ when $R>1$.  For example, in the case of H{\"o}lder continuous functions of order $\alpha\in (0,1]$ with a constant $\lambda>0$ on $E\subseteq [-R,R]^d$, our approximation rate becomes $3\lambda (3R\sqrt{d}\,{2^{-N}})^\alpha$. 
}

{Second, most interesting continuous functions in practice have a good modulus of continuity such that there is no implicit curse of dimensionality hidden in $\omega_f(\cdot)$. For example, we have discussed the case of H{\"o}lder continuous functions previously. We would like to remark that the class of H{\"o}lder continuous functions implicitly depends on $d$ through its definition in Equation~\eqref{eqn:Holder}, but this dependence is moderate since the $\ell^2$-norm in Equation~\eqref{eqn:Holder} is the square root of a sum with $d$ terms.  Let us now discuss several cases of $\omega_f(\cdot)$ when we cannot achieve exponential convergence or cannot avoid the curse of dimensionality. The first example is $\omega_f(r)=\tfrac{1}{\ln (1/r)}$ for small $r> 0$, which leads to an approximation rate
	\begin{equation*}
		3(N\ln 2-\tfrac12\ln d-\ln 2)^{-1},\quad \tn{for large $N\in\N^+$}.
	\end{equation*} 
	Apparently, the above approximation rate still avoids the curse of dimensionality but there is no exponential convergence, which has been canceled out by ``$\ln$'' in $\omega_f(\cdot)$. The second example is $\omega_f(r)=\tfrac{1}{\ln^{1/d} (1/r)}$ for  small $r> 0$, which leads to an approximation rate 
	\begin{equation*}
		3(N\ln 2-\tfrac12\ln d-\ln 2)^{-1/d},\quad \tn{for large $N\in\N^+$}.
	\end{equation*} 
	The power ${1}/{d}$ further weaken{s} the approximation rate and hence the curse of dimensionality exists. The last example we would like to discuss is $\omega_f(r)=r^{\alpha/d}$ for  small $r> 0$, which results in the approximation rate
	\[
	3\lambda (2\sqrt{d})^{{\alpha}/{d}}{2^{-{\alpha   N}/{d}}},\quad \tn{for large $N\in\N^+$}, 
	\]
	which achieves the exponential convergence and avoids the curse of dimensionality when we use very wide networks. 
	Though we have provided several examples of immoderate $\omega_f(\cdot)$, to the best of our knowledge, we are not aware of useful continuous functions with $\omega_f(\cdot)$ that is immoderate.}

\black{Finally, we would like to point out that the training of FLES networks in practice may encounter two issues. First, network weights in our main theorems require high-precision computation that might not be available in existing computer systems when the dimension $d$ and the network size parameter $N$ are large. But there is no theoretical evidence to exclude the possibility that similar approximation results can be achieved with reasonable weights in practical computation. Second, the vanishing gradient of piecewise constant activation functions makes standard SGD infeasible. There are two possible directions to solve the optimization problem for FLES networks: 1) gradient-free optimization methods, e.g., Nelder-Mead method \cite{Nelder:1965zz}, genetic algorithm \cite{10.2307/24939139}, simulated annealing \cite{Kirkpatrick671}, particle swarm optimization \cite{488968}, and consensus-based optimization \cite{doi:10.1142/S0218202517400061,Carrillo2019ACG}; 2) applying optimization algorithms for quantized networks that also have piecewise constant activation functions \cite{8955646,Boo2020QuantizedNN,2013arXiv1308.3432B,qnn2,qnn,Yin2019UnderstandingSE}. For example, an empirical way is to use a straight-through estimator (STE) by
setting the incoming gradients to the activation function (e.g., Floor) equal to its outgoing gradients, disregarding the derivative of the activation function itself. It would be interesting future work to explore efficient learning algorithms based on the FLES network.
}

\subsection{Kolmogorov-Arnold Superposition Theorem}

A closely related research topic is the Kolmogorov-Arnold representation theorem (KST) \cite{kolmogorov1956,arnold1957,kolmogorov1957} and its approximation in a form of modern neural networks. Our FLES networks admit super approximation power with a fixed number of layers for continuous functions and the KST exactly represent continuous functions using two hidden layers and $\calO(d)$ neurons. More specifically, given any $f\in C([0,1]^d)$, the KST shows that there exist continuous functions $\phi_q:\R\to \R$ and $\psi_{q,p}:[0,1]\to\R$ such that
\begin{equation}\label{eq:kst}
	f(\bmx)=\sum_{q=0}^{2d}\phi_q \bigg(\sum_{p=1}^{d} \psi_{q,p}(x_p)\bigg),\quad \tn{for any $\bmx=(x_1,\cdots,x_d)\in [0,1]^d$.}
\end{equation}
Note that the activation functions $\{\phi_q\}$ (also called outer functions) of the neural network in Equation~\eqref{eq:kst} have to depend on the target function $f$, though $\{\psi_{q,p}\}$ (also called inner functions) can be independent of $f$. The modulus of continuity of $\{\psi_{q,p}\}$ can be constructed such that they moderately depend on $d$, but the modulus of continuity of $\{\phi_q\}$ would be exponentially bad in $d$. In sum, the outer functions are too pathological such that there is no existing numerical algorithms to evaluate these activation functions, even though they are shown to exist by iterative construction \cite{braun2009}. 

There has been an active research line to develop more practical network approximation based on KST \cite{vera,kurkova1992, MAIOROV199981,GUL,MO,igelnik2003,schmidthieber2020kolmogorovarnold} by relaxing the exact representation to network approximation with an $\varepsilon$-error. The key issue these KST-related networks attempting to address is the $f$-dependency of the activation functions and the main goal is to construct neural networks conquering the curse of dimensionality in a more practical way computationally. The main idea of these variants is to apply computable activation functions independent of $f$ to construct neural networks to approximate the outer and inner functions of the KST, resulting in a larger network that can approximate a continuous function with the desired accuracy. Using this idea, the seminal work in \cite{kurkova1992} applied sigmoid activation functions and constructed two-hidden-layer networks to approximate $f\in C([0,1]^d)$. Though the activation functions are independent of $f$, the number of neurons scales exponentially in $d$ and the curse of dimensionality exists. Cubic-splines and piecewise linear functions have also been used to approximate the outer and inner functions of KST in \cite{igelnik2003,MO,schmidthieber2020kolmogorovarnold}, resulting in cubic-spline networks or deep ReLU networks to approximate $f\in C([0,1]^d)$. But the approximation bounds in these works still suffer from the curse of dimensionality unless $f$ has simple outer functions in the KST. It is still an open problem to characterize the class of functions with a moderate outer function in KST.

To the best of our knowledge, the most successful construction of neural networks with $f$-independent activation functions conquering the curse of dimensionality is in \cite{MAIOROV199981,GUL}, where a two-hidden-layer network with $\calO(d)$ neurons can approximate $f\in C([0,1]^d)$ within an arbitrary error $\varepsilon$. Let us briefly summerize their main ideas to obtain such an \black{exciting} result here. 1) Identify a dense and countable subset $\{u_k\}_{k=1}^\infty$ of $C([-1,1])$, e.g., polynomials with rational coefficients. 2) Construct an activation function $\varrho$ to ``store'' all $u_k(x)$ for $x\in[-1,1]$. For example, divide the domain of $\varrho(x)$ into countable pieces and each piece is a connected interval of length $2$ associated with a $u_k$. In particular, let $\varrho(x+4k+1)=a_k+b_k x+c_k u_k(x)$ for any $x\in [-1,1]$ with carefully chosen constants $a_k$, $b_k$, and $c_k$ such that $\varrho(x)$ can be a sigmoid function. 3) By construction, there exists a one-hidden-layer network with width $3$ and $\varrho(x)$ as the activation function to approximate any outer or inner function in KST with an arbitrary accuracy parameter $\delta$. Only the parameters of the one-hidden-layer network depend on the target function and accuracy. 4) Replace the inner and outer function in KST with these one-hidden-layer networks to achieve a two-hidden-layer network with $\varrho(x)$ as the activation function and width $\calO(d)$ to approximate an arbitrary $f\in C([0,1]^d)$ within an arbitrary error $\varepsilon$. Unfortunately, the construction of the parameters of this magic network relies on the evaluation of the outer and inner functions of KST, which is not computationally feasible even if computation with arbitrary precision is allowed. 

We would like to remark that, though the approximation rate of FLES networks in this paper is relatively worse than the approximation rate in \cite{MAIOROV199981,GUL}, our activation functions are much simpler and there are explicit formulas to specify the parameters of FLES networks. If computation with an arbitrary precision is allowed and the target function $f$ can be arbitrarily sampled, we can specify all the weights in FLES networks. Besides, our approximation rate is sufficiently attractive since it is exponential and avoids the curse of dimensionality. For a large dimension $d$, the width parameter of our FLES network can be chosen as $N=d$, which leads to a FLES network of size $\calO(d)$ with an approximation accuracy $\calO(2^{-d})$ for Lipschitz continuous functions. $\calO(2^{-d})$ is sufficiently attractive. In practice, when $d$ is very large, $N$ could be much smaller than $d$ and our approximation rate is still attractive.  

Finally, we list several KST-related results in Table~\ref{tab:kstLike} for a quick comparison.\footnote{The result in \cite{Shen1} is for H\"older functions, but can be easily generalized to general continuous functions. } 
\black{As shown in Table~\ref{tab:kstLike}, there exists a trade-off between the complexity of activation functions and the network size when the approximation error is fixed. A key advantage of our FLES networks is to use simple and  explicit activation functions to attain an exponential convergence rate.}

\begin{table}    
	\caption{ A comparison of several KST-related results for approximating $f\in C([0,1]^d)$.} 
	\label{tab:kstLike}
	\centering 
	\resizebox{0.985\textwidth}{!}{ %
		\begin{tabular}{cccccccc} 
			\toprule
			paper  & number of hidden layers & width &  activation function(s) & error & remark \\
			
			\midrule
			\cite{kolmogorov1956,arnold1957,kolmogorov1957}  & $2$ & $2d+1$ & 
			$f$-dependent &  0  & original KST\\
			
			\midrule
			\cite{MAIOROV199981,GUL}  & $2$ & $\calO(d)$ & 
			$f$-independent &  arbitrary error $\varepsilon$ & based on KST\\
			
			\midrule
			\cite{Shen1}  & $3$ & $\calO(dN)$ & 
			ReLU & $\calO(\omega_f(N^{-2/d}))$  & not based on KST\\
			
			\midrule
			this paper  & $3$ & $\max\{d,N\}$ & 
			($\sigma_1,\sigma_2,\sigma_3)$ & $2\omega_f(\sqrt{d})2^{-N}+\omega_{f}(\sqrt{d}\,2^{-N})$  & not based on KST\\
			
			\bottomrule
		\end{tabular} 
	}%
\end{table}

\subsection{Discussion on the literature}

In this section, we will discuss other recent development of neural network approximation. Our discussion will be divided into mainly three parts according to the analysis methodology in the references: 1) functions admitting integral representations; 2) linear approximation; 3) bit extraction.

In the seminal work of \cite{barron1993}, its variants or generalization \cite{barron2018approximation,Weinan2019,doi:10.1002/mma.5575,bandlimit}, and related references therein,  $d$-dimensional functions of the following form were considered:
\begin{equation*}
	f(\bmx)=\int_{\widetilde{\Omega}} a(\bm{w}) K(\bm{w}\cdot\bmx)d\mu(\bm{w}),
\end{equation*}
where $\widetilde{\Omega}\subseteq\mathbb{R}^d$, $\mu(\bm{w})$ is a Lebesgue measure in $\bm{w}$, and $\bm{x}\in \Omega\subseteq\mathbb{R}^d$. The above integral representation is equivalent to the expectation of a high-dimensional random function when $\bm{w}$ is treated as a random variable. By the law of large number theory, the average of $N$ samples of the integrand leads to an approximation of $f(\bmx)$ with an approximation error bounded by $\frac{C_f \sqrt{\mu(\Omega)}}{\sqrt{N}}$ measured in $L^2(\Omega,\mu)$ (Equation (6) of \cite{barron1993}), where $\calO(N)$ is the total number of parameters in the network, $C_f$ is a $d$-dimensional integral with an integran{d} related to $f$, and $\mu(\Omega)$ is the Lebesgue measure of $\Omega$. As discussed in \cite{barron1993}, $\mu(\Omega)$ and $C_f$ would be exponential in $d$ and standard smoothness properties of $f$ alone are not enough to remove the exponential dependence of $C_f$ on $d$. Therefore, the curse of dimensionality exists in the whole approximation error while the curse does not exist in the approximation rate in $N$.


Linear approximation is an efficient approximation tool for smooth functions that computes the approximant of a target function via a linear projection to a Hilbert space or a Banach space as the approximant space. Typical examples include approximation via orthogonal polynomials, Fourier series expansion, etc. Inspired by the seminal work in \cite{yarotsky2017}, where deep ReLU networks were constructed to approximate polynomials with exponential convergence, subsequent works in \cite{EWang,Opschoor2019,Hadrien,doi:10.1002/mma.5575,bandlimit,yarotsky2019,Shen3,MO,Wang} have constructed deep ReLU networks to approximate various smooth function spaces. The main idea of these works is to approximate smooth functions via (piecewise) polynomial approximation first and then construct deep ReLU networks to approximate the ensemble of polynomials. 
But the curse of dimensionality exists since polynomial approximation cannot avoid the curse. Finally, a different approach is used  in   \cite{2019arXiv191210382L}. The authors  of \cite{2019arXiv191210382L} use   a dynamic system based approach to obtain a universal approximation property of residual neural networks.

The bit extraction proposed in   \cite{Bartlett98almostlinear} has been a very important technique to develop nearly optimal approximation rates of deep ReLU neural networks \cite{yarotsky18a,Shen2,Shen3,Wang,shijun:thesis,2021arXiv210300502S} and the optimality is based on the nearly optimal VC-dimension bound of ReLU networks in \cite{pmlr-v65-harvey17a}. The bit extraction was also applied in \cite{Shen4,schmidthieber2020kolmogorovarnold} and this paper to develop network approximation theories. In the first step, an efficient projection map in a form of a ReLU, or a Floor-ReLU, or a FLES network is constructed to project high-dimensional points to one-dimensional points such that the high-dimensional approximation problem is reduced to a one-dimensional approximation problem.  In the second step, the one-dimensional approximation problem is solved by constructing a ReLU, or a Floor-ReLU, or a FLES network, which can be efficiently compressed via the bit extraction. Although shallower neural networks can also carry out the above two steps, bit extraction can take full advantage of the power of depth and construct deep neural networks with a nearly optimal number of parameters or neurons to fulfill the above two steps.

\section{Theoretical Analysis}
\label{sec:approxContFunc}
In this section, we first introduce basic notations in this paper in Section~\ref{sec:notation}. Then we prove Theorem~\ref{thm:main} and Corollary~\ref{cor:subsetE} in Section~\ref{sec:proof:main} based on Theorem~\ref{thm:mainOld}, which is proved in Section~\ref{sec:proof:main:old}.

\subsection{Notations}
\label{sec:notation}

The main notations of this paper are listed as follows.
\begin{itemize}    
	\item Vectors and matrices are denoted in a bold font. Standard vectorization is adopted in the matrix and vector computation. For example, adding a scalar and a vector means adding the scalar to each entry of the vector.

	\item Let $\R$  denote the set of real numbers.
	
	\item Let $\Z$, $\N$, and $\N^+$ denote the set of integers, natural numbers,  all positive integers, respectively, i.e., $\Z=\{0,1,2,\cdots\}\cup\{-1,-2,-3,\cdots\}$, $\N=\{0,1,2,\cdots\}$, and $\N^+=\{1,2,3,\cdots\}$.   
	
	\item For any $p\in [1,\infty)$, the $p$-norm of a vector $\bmx=(x_1,x_2,\cdots,x_d)\in\R^d$ is defined by 
	\begin{equation*}
		\|\bmx\|_p\coloneqq \big(|x_1|^p+|x_2|^p+\cdots+|x_d|^p\big)^{1/p}.
	\end{equation*}
	
	
	\item The floor function (Floor) is defined as $\lfloor x\rfloor:=\max \{n: n\le x,\ n\in \mathbb{Z}\}$ for any $x\in \R$. 
	
	
	
	\item For $\theta\in[0,1)$, suppose its binary representation is $\theta=\sum_{\ell=1}^{\infty}\theta_\ell2^{-\ell}$ with $\theta_\ell\in \{0,1\}$, we introduce a special notation $\bin 0.\theta_1\theta_2\cdots \theta_L$ to denote the $L$-term binary representation of $\theta$, i.e., ${\bin 0.\theta_1\theta_2\cdots \theta_L\coloneqq}\sum_{\ell=1}^{L}\theta_\ell2^{-\ell}$. 
	
	
	\item The expression ``a network with  width $N$ and  depth $L$'' means
	\begin{itemize}
		\item The maximum width of this network for all \textbf{hidden} layers  is no more than $N$.
		\item The number of \textbf{hidden} layers of this network is no more than $L$.
	\end{itemize}
\end{itemize}

\subsection{Proof of Theorem~\ref{thm:main}  and Corollary~\ref{cor:subsetE}}\label{sec:proof:main} 
In this section, we will prove Theorem~\ref{thm:main}  and Corollary~\ref{cor:subsetE}. 
To this end, we first introduce Theorem~\ref{thm:mainOld} that works only for $[0,1)^d$,  regraded as a weaker variant of Theorem~\ref{thm:main}.
\begin{theorem}\label{thm:mainOld}
	Given  an arbitrary continuous function $f$ on $[0,1]^d$, for any  $N\in \N^+$, there exist $a_1,a_2,\cdots,a_ N\in[0,\tfrac12)$ such that
	\begin{equation*}
		|\phi(\bmx)-f(\bmx)|\le 2\omega_f(\sqrt{d})2^{-N}+\omega_{f}(\sqrt{d}\,2^{-N}),
	\end{equation*}
	for any $\bmx=(x_1,x_2,\cdots,x_d)\in [0,1)^d$, where $\phi$ is defined by a formula in $a_1,a_2,\cdots,a_N$ as follows. 
	\begin{equation*}
			\mathResize[0.95]{
				\phi(\bmx)=2\omega_f(\sqrt{d})\sum_{j=1}^ N 2^{-j} \sigma_3\bigg(\, a_j\cdot\sigma_2\Big(\, 1+\sum_{i=1}^{d} 2^{(i-1)N}\sigma_1(2^N x_i)\, \Big)\, \bigg)+f(\bmzero)-\omega_f(\sqrt{d}).
			}
	\end{equation*}
\end{theorem}
We will prove Theorem~\ref{thm:main}  and Corollary~\ref{cor:subsetE} based on Theorem~\ref{thm:mainOld}, the proof of which can be found in Section~\ref{sec:proof:main:old}.

\vspace{8pt}
First, let us prove Theorem~\ref{thm:main} by assuming Theorem~\ref{thm:mainOld} is true.
\begin{proof}[Proof of Theorem~\ref{thm:main}]
	Given any $f\in C([0,1]^d)$, by Lemma $4.2$ of \cite{Shen2} via setting $E=[0,1]^d$ and $S=\R^d$, there exists $g\in C(\R^d)$ such that 
	\begin{itemize}
		\item $g(\bmx)=f(\bmx)$ for any $\bmx\in E=[0,1]^d$;
		\item $\omega_g^{S}(r)=\omega_f^{E}(r)=\omega_f(r)$ for any $r\ge 0$.
	\end{itemize} 
	
	Define $\tildeg(\bmx)\coloneqq  g(2\bmx)$ for any $\bmx\in \R^d$.
	By applying Theorem~\ref{thm:mainOld} to $\tildeg\in C(\R^d)$, there exist $a_1,a_2,\cdots,a_ N\in[0,\tfrac12)$ such that
	\begin{equation}\label{eq:phi-g:one}
		|\tildephi(\bmx)-\tildeg(\bmx)|\le 2\omega_\tildeg^{S}(\sqrt{d})2^{-N}+\omega_{\tildeg}^{S}(\sqrt{d}\,2^{-N}),\quad \tn{for any $\bmx\in [0,1)^d$,}
	\end{equation}	
	where 
	\begin{equation*}
			\mathResize[0.95]{
				\tildephi(\bmx)=2\omega_\tildeg^S(\sqrt{d})\sum_{j=1}^ N 2^{-j} \sigma_3\bigg(\, a_j\cdot\sigma_2\Big(\, 1+\sum_{i=1}^{d} 2^{(i-1)N}\sigma_1(2^N x_i)\, \Big)\, \bigg)+\tildeg(\bmzero)-\omega_\tildeg^S(\sqrt{d}).
			}
	\end{equation*}

	Note that $f(\bmx)=g(\bmx)=\tildeg(\tfrac{\bmx}{2})$ for any $\bmx\in E=[0,1]^d$ and
	\[\omega_\tildeg^{S}(r)=\omega_g^{S}(2r)=\omega_f^E(2r)=\omega_f(2r),\quad \tn{ for any $r\ge 0$}.\]
	Define $\phi(\bmx)\coloneqq\tildephi(2\bmx)$ for any $\bmx\in\R^d$.
	 Then by Equation~\eqref{eq:phi-g:one}, for any $\bmx\in [0,1]^d=E$, we have $\tfrac{\bmx}{2}\in [0,\tfrac12]^d\subseteq [0,1)^d$, implying
	\begin{equation*}
		\begin{split}
			|\phi(\bmx)-f(\bmx)|=|\phi(\bmx)-g(\bmx)|  &=|\tildephi(\tfrac{\bmx}{2})-\tildeg(\tfrac{\bmx}{2})|\\
			&\le 2\omega_\tildeg^{S}(\sqrt{d})2^{-N}+\omega_{\tildeg}^{S}(\sqrt{d}\,2^{-N})\\
			&=2\omega_f(2\sqrt{d})2^{-N}+\omega_{f}(2\sqrt{d}\,2^{-N}),
		\end{split}
	\end{equation*}
	where $\phi(\bmx)\coloneqq \tildephi(\tfrac{\bmx}{2})$ can be represented by 
	\begin{equation*}
			\mathResize[0.95]{
			2\omega_f(2\sqrt{d})\sum_{j=1}^ N 2^{-j} \sigma_3\bigg(\, a_j\cdot\sigma_2\Big(\, 1+\sum_{i=1}^{d} 2^{(i-1)N}\sigma_1(2^{N-1}x_i)\, \Big)\, \bigg)
			+f(\bmzero)-\omega_f(2\sqrt{d}).
			}
	\end{equation*}
	With the discussion above, we have proved Theorem~\ref{thm:main}.
\end{proof}

Next, we present the proof of Corollary~\ref{cor:subsetE}  below.
\begin{proof}[Proof of Corollary~\ref{cor:subsetE}]
	Given any bounded continuous function $f\in C(E)$, by Lemma $4.2$ of \cite{Shen2} via setting $S=\R^d$, there exists $g\in C(\R^d)$ such that 
	\begin{itemize}
		\item $g(\bmx)=f(\bmx)$ for any $\bmx\in E\subseteq [-R,R]^d$;
		\item $\omega_g^{S}(r)=\omega_f^{E}(r)$ for any $r\ge 0$.
	\end{itemize} 
	
	Define 
	\[\tildeg(\bmx)\coloneqq   g(3R\bmx-R),\quad \tn{for any $\bmx\in \R^d$.}\] 
	By applying Theorem~\ref{thm:mainOld} to $\tildeg\in  C(\R^d)$, there exist $a_1,a_2,\cdots,a_ N\in[0,\tfrac12)$ such that
	\begin{equation}\label{eq:phi-g:two}
		|\tildephi(\bmx)-\tildeg(\bmx)|\le 2\omega_\tildeg^{S}(\sqrt{d})2^{-N}+\omega_{\tildeg}^{S}(\sqrt{d}\,2^{-N}),\quad \tn{for any $\bmx\in [0,1)^d$,}
	\end{equation}
	where 
	\begin{equation*}
			\mathResize[0.95]{
			\tildephi(\bmx)=2\omega_\tildeg^S(\sqrt{d})\sum_{j=1}^ N 2^{-j} \sigma_3\bigg(\, a_j\cdot\sigma_2\Big(\, 1+\sum_{i=1}^{d} 2^{(i-1)N}\sigma_1(2^{N} x_i)\, \Big)\, \bigg)+\tildeg(\bmzero)-\omega_\tildeg^S(\sqrt{d}).}
	\end{equation*}

	Note that  $f(\bmx)=g(\bmx)=\tildeg(\tfrac{\bmx+R}{3R})$ for any $\bmx\in E\subseteq [-R,R]^d$ and
	\[\omega_\tildeg^{S}(r)=\omega_g^{S}(3Rr)=\omega_f^E(3Rr),\quad \tn{ for any $r\ge 0$}.\]
	Define $\phi(\bmx)\coloneqq\tildephi(\tfrac{\bmx+R}{3R})$ for any $\bmx\in\R^d$.
	Then by Equation~\eqref{eq:phi-g:two}, for any $\bmx\in E\subseteq [-R,R]^d$, we have $\tfrac{\bmx+R}{3R}\in [0,\tfrac23]^d\subseteq [0,1)^d$, implying
	\begin{equation*}
		\begin{split}
			|\phi(\bmx)-f(\bmx)|=|\phi(\bmx)-g(\bmx)|  
			&=|\tildephi(\tfrac{\bmx+R}{3R})-\tildeg(\tfrac{\bmx+R}{3R})|\\
			&\le 2\omega_\tildeg^{S}(\sqrt{d})2^{-N}+\omega_{\tildeg}^{S}(\sqrt{d}\,2^{-N})\\
			&=2\omega_f(3R\sqrt{d})2^{-N}+\omega_{f}(3R\sqrt{d}\,2^{-N}),
		\end{split}
	\end{equation*}
	where $\phi(\bmx)=\tildephi(\tfrac{\bmx+R}{3R})$ can be represented by 
	\begin{equation*}		
			\mathResize[0.99]{
			2\omega_f(3R\sqrt{d})\sum\limits_{j=1}^ N 2^{-j} \sigma_3\bigg(\, a_j\cdot\sigma_2\Big(\, 1+\sum\limits_{i=1}^{d} 2^{(i-1)N}\sigma_1(2^N\tfrac{x_i+R}{3R})\, \Big)\, \bigg)
			+C_f,
			}
	\end{equation*}
	where $C_f=\tildeg(\bmzero)-\omega_\tildeg^S(\sqrt{d})$ is a constant essentially determined by $f$.
	With the discussion above, we have proved Corollary~\ref{cor:subsetE}.
\end{proof}

\subsection{Proof of Theorem~\ref{thm:mainOld}}\label{sec:proof:main:old}

To prove Theorem~\ref{thm:mainOld}, we first present the proof sketch.
Shortly speaking, we construct piecewise constant functions to approximate continuous functions. There are five key steps in our construction.
\begin{enumerate}
	\item Normalize $f$ as $\tildef$ satisfying $\tildef(\bmx)\in[0,1]$ for any $\bmx\in [0,1]^d$, divide $[0,1)^d$ into a set of non-overlapping cubes $\{Q_\bmbeta\}_{\bmbeta\in \{0,1,\cdots,J-1\}^d}$, and denote $\bmx_\bmbeta$ as the vertex of $Q_\bmbeta$ with minimum $\|\cdot\|_1$ norm,  where $J$ is an integer determined later. See Figure~\ref{fig:Qbeta+xbeta} for the illustrations of $Q_\bmbeta$ and $\bmx_\bmbeta$.
	
	\item Construct a vector-valued function $\bm{\Phi}_1:\R^d\to \R^d$ projecting  the whole cube $ Q_\bmbeta$ to the index $\bmbeta$, i.e., $\bmPhi_1(\bmx)=\bmbeta$ for all $\bmx\in Q_\bmbeta$ and each $\bmbeta\in \{0,1,\cdots,J-1\}^d$.

	\item Construct a linear function $\phi_2:\R^d\to \R$ bijectively mapping $\bmbeta\in\{0,1,\cdots,J-1\}^d$ to $\phi_2(\bmbeta)\in \{1,2,\cdots,J^d\}$.
	
	\item Construct a function $\phi_3:\R\to \R$ mapping $\phi_2(\bmbeta)\in\{1,2,\cdots,J^d\}$ approximately to $\tildef(\bmx_\bmbeta)$, i.e., $\phi_3(\phi_2(\bmbeta))\approx\tildef(\bmx_\bmbeta)$ for each $\bmbeta\in \{0,1,\cdots,J-1\}^d$.
	
	\item Define $\tildephi\coloneqq \phi_3\circ\phi_2\circ\bm{\Phi}_1$. Then $\tildephi$ is a piecewise constant function mapping $\bmx\in Q_\bmbeta$ to $\phi_3(\phi_2(\bmbeta))\approx \tildef(\bmx_\bmbeta)$ for each $\bmbeta\in \{0,1,\cdots,J-1\}^d$, implying $\tildephi\approx \tildef$. Finally, re-scale and shift $\tildephi$ to obtain the final function $\phi$ approximating $f$ well.
\end{enumerate}

Recall that
\begin{equation*}
	\sigma_1(x)\coloneqq\lfloor x\rfloor,\quad \sigma_2(x)\coloneqq 2^x,\quad \tn{and}\quad \sigma_3(x)\coloneqq\calT(x-\lfloor x\rfloor-\tfrac12),\quad \tn{for any $x\in \R$,}
\end{equation*}
where \[\calT(x)\coloneqq \one_{x\geq 0} =
\left\{\genfrac{}{}{0pt}{0}{
	1,\   x\ge 0,}
	{0,\  x<0,}
	\right.\quad \tn{ for any $x\in \R.$}\]
Step $1$ and $5$ are straightforward. To implement Step $2$, we introduce $\sigma_1$ since it can help to significantly simplify the construction of the vector-valued projecting function $\bmPhi_1$. The implementation of Step $3$ is based on the $J$-ary representation, namely, define $\phi_2(\bmx)\coloneqq 1+\sum_{i=1}^{d}J^{i-1} x_i$.
The most technical step above is Step $4$, which is essentially a point fitting problem. Solving such a problem  eventually relies on  the bit extraction technique in \cite{Shen2,Shen3,Shen4,pmlr-v65-harvey17a,Bartlett98almostlinear,shijun:thesis,yarotsky18a}.
To extract sufficient many bits with a limited neuron budget, we introduce two powerful activation functions $\sigma_2$ and $\sigma_3$,
as shown in the proposition below.
\begin{proposition}
	\label{prop:bitsExtract}
	Given any $K\in \N^+$ and arbitrary $\theta_1,\theta_2,\cdots,\theta_K\in \{0,1\}$, it holds that
	\begin{equation*}
		\sigma_3\big(a\cdot \sigma_2(k)\big)=\sigma_3(2^k\cdot a)=\theta_k,\quad \tn{for any $k\in\{1,2,\cdots,K\}$,}
	\end{equation*}
	where 
	\begin{equation*}
		a=\sum_{j=1}^{K} 2^{-j-1}\cdot\theta_j\ \in [0,\tfrac12).
	\end{equation*}
\end{proposition}
\begin{proof}
	
	Since $\theta_j\in\{0,1\}$ for $j\in\{1,2,\cdots,K\}$,  we have \[0\le \sum_{j=1}^{K} 2^{-j-1}\cdot\theta_j\le \sum_{j=1}^{K} 2^{-j-1}<\tfrac12,\] implying
	$a\in[0,\tfrac12)$.
	
	Next, fix $k\in\{1,2,\cdots,K\}$ for the proof below. It holds that
	\begin{equation}\label{eq:3term}
		2^k\cdot a=2^k\cdot \sum_{j=1}^{K} 2^{-j-1}\cdot\theta_j=\underbrace{\sum_{j=1}^{k-1}2^{k-j-1}\cdot\theta_j}_{\tn{an integer}}\ +\
		\overbrace{\tfrac12 \theta_k}^{\tn{$0$ or $\tfrac12$}}\ +\  \underbrace{\sum_{j=k+1}^{K}2^{k-j-1}\cdot\theta_j}_{\tn{in $[0,\tfrac12)$}}.\footnote{By convention, $\sum_{j=n}^m a_j=0$ if $n>m$ no matter what $a_j$ is for each $j$.}
	\end{equation}
	Clearly, the first term in Equation~\eqref{eq:3term} $\sum_{j=1}^{k-1}2^{k-j-1}\cdot\theta_j$ is a non-negative integer since $\theta_j\in\{0,1\}$ for any $j\in \{1,2,\cdots,K\}$. As for the third term in Equation~\eqref{eq:3term}, we have
	\begin{equation*}
		0\le \sum_{j=k+1}^{K}2^{k-j-1}\cdot\theta_j\le \sum_{j=k+1}^{K}2^{k-j-1}<\tfrac12
	\end{equation*}
	Therefore, by Equation~\eqref{eq:3term}, we have
	\begin{equation}\label{eq:2case}
		2^k\cdot a\in \bigcup_{n\in\N}[n,n+\tfrac12), \tn{ if $\theta_k=0$}\quad \tn{and}\quad 2^k\cdot a\in \bigcup_{n\in\N}[n+\tfrac12,n+1), \tn{ if $\theta_k=1$.}
	\end{equation}
	Recall that $\sigma_3(x)=\calT(x-\lfloor x\rfloor-\tfrac12)$, where $\calT(x)=\left\{\begin{smallmatrix*}[l]
		1,\ x\ge0,\\ 0,\ x<0\phantom{.}
	\end{smallmatrix*}\right.$. It is easy to verify that 
	\begin{equation*}
		\tn{$\sigma_3(x)=0$ if $x\in \bigcup_{n\in\N}[n,n+\tfrac12)$\quad and \quad $\sigma_3(x)=1$ if $x\in \bigcup_{n\in\N}[n+\tfrac12,n+1)$.}
	\end{equation*}	
	If $\theta_k=0$, by Equation~\eqref{eq:2case}, we have 
	\begin{equation*}
		2^k\cdot a\in  \bigcup_{n\in\N}[n,n+\tfrac12)\quad \Longrightarrow\quad  \sigma_3(2^k\cdot a)=0=\theta_k.
	\end{equation*}
	Similarly, if $\theta_k=1$, by Equation~\eqref{eq:2case}, we have 
	\begin{equation*}
		2^k\cdot a\in  \bigcup_{n\in\N}[n+\tfrac12,n+1)\quad \Longrightarrow\quad  \sigma_3(2^k\cdot a)=1=\theta_k.
	\end{equation*}
	
	Since $k\in\{1,2,\cdots,K\}$ is arbitrary, we have $\sigma_3\big(a\cdot \sigma_2(k)\big)=\sigma_3(2^k\cdot a)=\theta_k$ for any $k\in\{1,2,\cdots,K\}$. So we finish the proof.
\end{proof}

We would like to point out that
Proposition~\ref{prop:bitsExtract} indicates that the VC-dimension of the function space
\[ \{f: f(x)=\sigma_3(a\cdot x),\ \tn{for} \ a\in\R\}\] is infinity,  \black{which implies that the VC-dimension of FLES networks is also infinity. As discussed previously in Section~\ref{sec:approx:vcdim}, having an infinite VC-dimension is a necessary condition for our FLES networks to attain super approximation power.}

\vspace{18pt}
With Proposition~\ref{prop:bitsExtract} in hand, we are ready to prove Theorem~\ref{thm:mainOld}.

\begin{proof}[Proof of Theorem~\ref{thm:mainOld}]
	The proof consists of five steps.
	\mystep{1}{Set up.}
	Assume $f$ is not a constant function since it is a trivial case. Then $\omega_f(r)>0$ for any $r>0$. Clearly, $|f(\bmx)-f(\bmzero)|\le \omega_f(\sqrt{d})$ for any $\bmx\in [0,1)^d$. Define 
	\begin{equation}
		\label{eq:tildef}
		\tildef\coloneqq\frac{f-f(\bmzero)+\omega_f(\sqrt{d})}{2\omega_f(\sqrt{d})}.
	\end{equation} 
	It follows that $\tildef(\bmx)\in [0,1]$ for any $\bmx\in [0,1)^d$.

	Set $J=2^N$ and divide $[0,1)^d$ into $J^d$ cubes $\{Q_\bmbeta\}_\bmbeta$. To be exact,  defined $\bmx_\bmbeta \coloneqq \bmbeta/J$ and 
	\begin{equation*}
		Q_\bmbeta\coloneqq \Big\{\bmx=(x_1,x_2,\cdots,x_d):x_i\in [\tfrac{\beta_i}{J},\tfrac{\beta_i+1}{J}) \tn{ for } i=1,2,\cdots,d\Big\},
	\end{equation*}
	for each $\bmbeta=(\beta_1,\beta_2,\cdots,\beta_d)\in \{0,1,\cdots,J-1\}^d$.
	See Figure~\ref{fig:Qbeta+xbeta} for illustrations.

	\begin{figure}[!htp]
		\centering
		\begin{subfigure}[b]{0.4\textwidth}
			\centering            \includegraphics[width=0.999\textwidth]{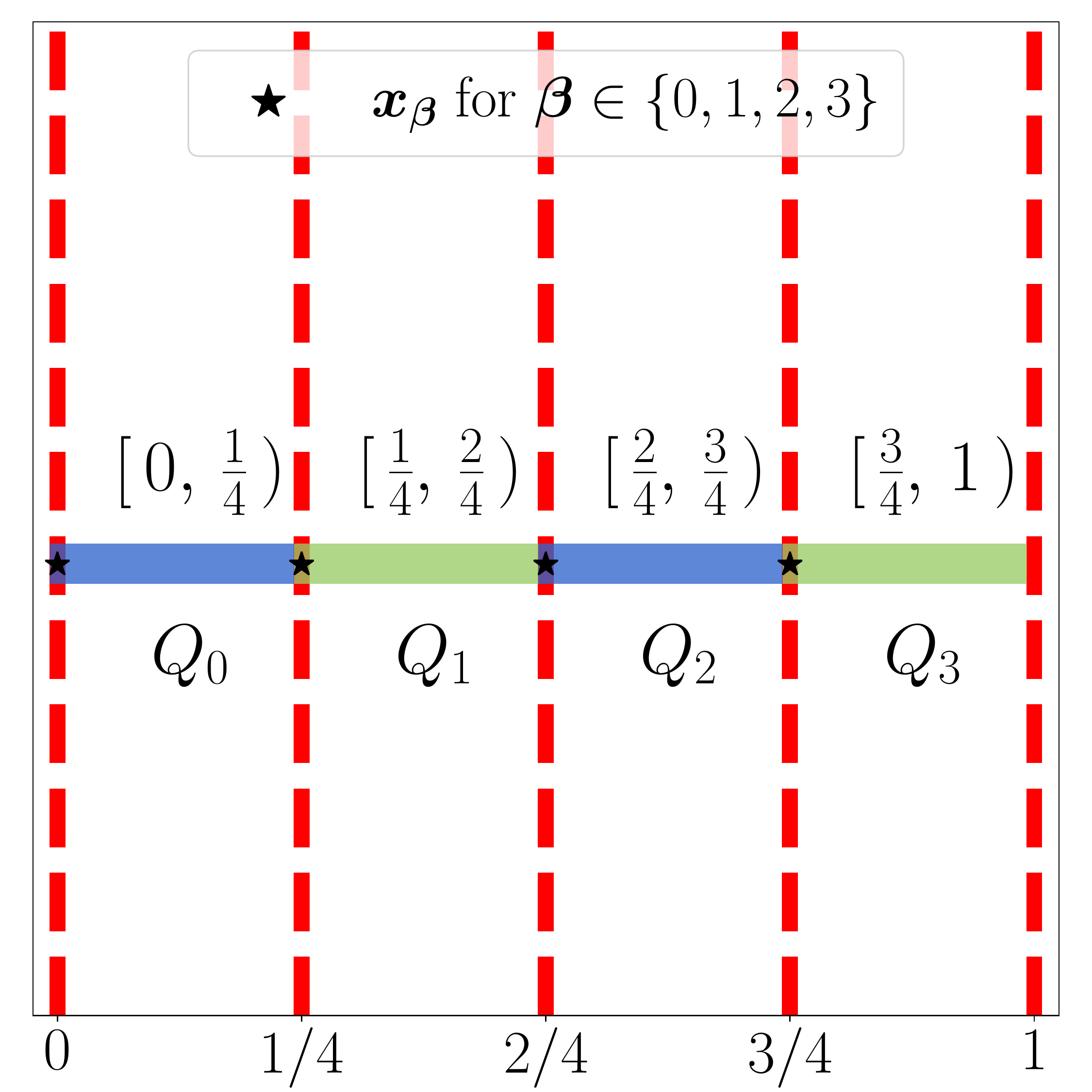}
			\subcaption{}
		\end{subfigure}
				\begin{minipage}{0.04\textwidth}
					\
				\end{minipage}
		\begin{subfigure}[b]{0.4\textwidth}
			\centering            \includegraphics[width=0.999\textwidth]{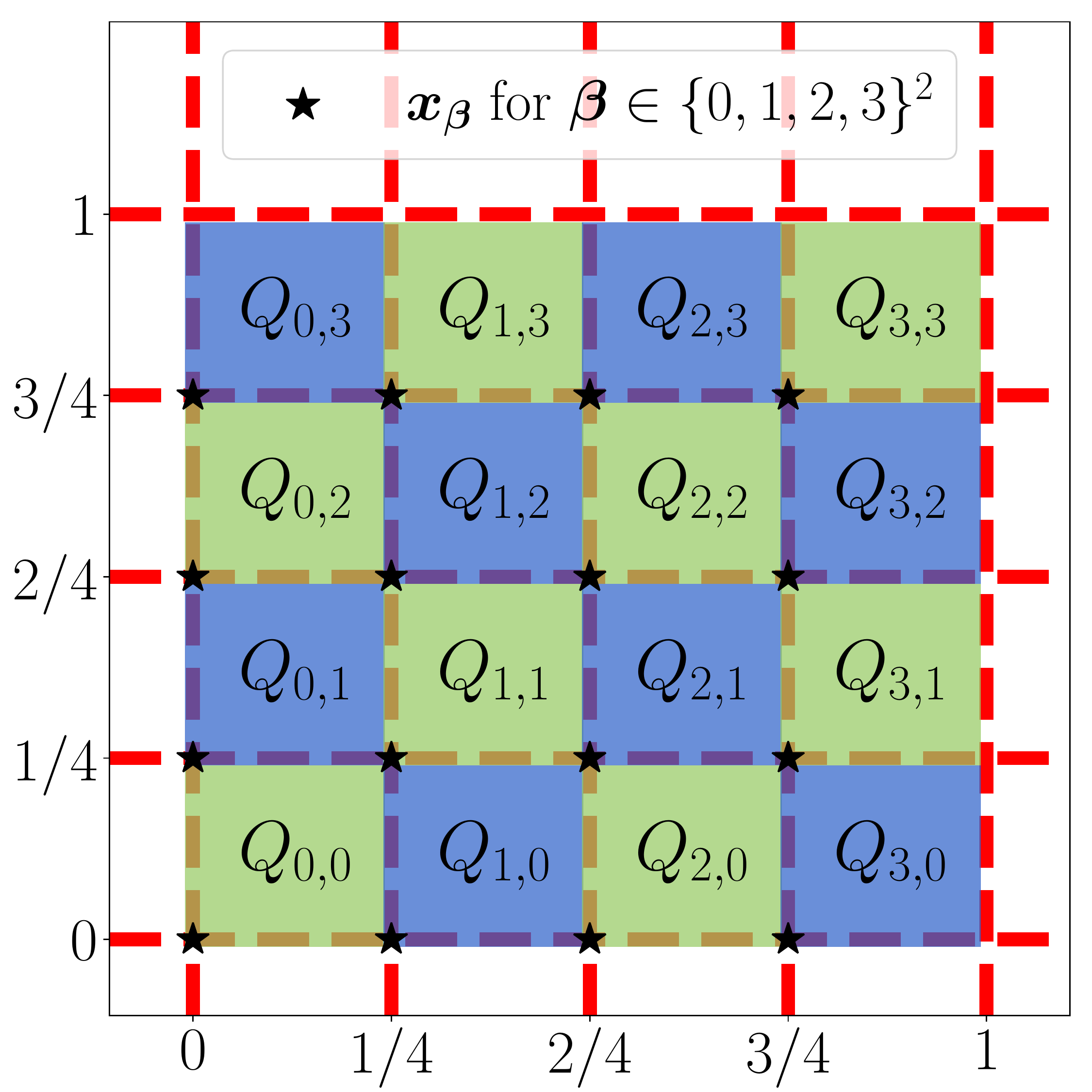}
			\subcaption{}
		\end{subfigure}
		\caption{Illustrations of $Q_\bmbeta$ and $\bmx_\bmbeta$ for any $\bmbeta\in \{0,1,\cdots,J-1\}^d$. (a) $J=4,\ d=1$. (b)  $J=4,\ d=2$.}
		\label{fig:Qbeta+xbeta}
	\end{figure}
	
	\mystep{2}{Construct $\bmPhi_1$ mapping $\bmx\in Q_\bmbeta$ to $\bmbeta$ for each $\bmbeta\in\{0,1,\cdots,J-1\}^d$.}
	Define 
	\begin{equation*}
		\bmPhi_1(\bmx)
		\coloneqq \Big(\sigma_1(J x_1), \sigma_1(J x_2), \cdots ,\sigma_1(J x_d) \Big)
		=\Big(\lfloor J x_1\rfloor, \lfloor J x_2\rfloor, \cdots ,\lfloor J x_d\rfloor \Big),
	\end{equation*}
	\tn{for any $\bmx=(x_1,x_2,\cdots,x_d)\in\R^d$.}
	Then, for any $\bmx\in Q_\bmbeta$ and each $\bmbeta\in\{0,1,\cdots,J-1\}^d$, we have
	\begin{equation}\label{eq:eq1}
		\bmPhi_1(\bmx)=\Big(\lfloor J x_1\rfloor, \lfloor J x_2\rfloor, \cdots ,\lfloor J x_d\rfloor \Big)=(\beta_1,\beta_2,\cdots,\beta_d)=\bmbeta.
	\end{equation}

	\mystep{3}{Construct $\phi_2$ bijectively mapping $\bmbeta\in\{0,1,\cdots,J-1\}^d$ to $\phi_2(\bmbeta)\in \{1,2,\cdots,J^d\}$.}
	Inspired by the $J$-ary representation, we define a linear function
	\begin{equation*}
		\phi_2(\bmx)\coloneqq 1+\sum_{i=1}^{d}J^{i-1} x_i,\quad \tn{for each $\bmx=(x_1,x_2,\cdots,x_d)\in\R^d$.}
	\end{equation*}
	Then $\phi_2$ is a bijection from $\{0,1,\cdots,J-1\}^d$ to $\{1,2,\cdots,J^d\}$.

	\mystep{4}{Construct $\phi_3$ mapping $\phi_2(\bmbeta)\in\{1,2,\cdots,J^d\}$ approximately to $\tildef(\bmx_\bmbeta)$.}
	
	For each $k\in\{1,2,\cdots,J^d\}$, there exists a unique $\bmbeta\in\{0,1,\cdots,J-1\}^d$ such that $\phi_2(\bmbeta)=k$. Thus, define
	\begin{equation}\label{eq:eq2}
		\xi_{k}\coloneqq \tildef(\bmx_\bmbeta)\in [0,1],\quad \tn{for any $k\in\{1,2,\cdots,J^d\}$ with $k=\phi_2(\bmbeta)$.}
	\end{equation}
	For each $k\in \{1,2,\cdots,J^d\}$, there exist $\theta_{k,1},\theta_{k,2},\cdots,\theta_{k,N}\in \{0,1\}$ such that
	\begin{equation}\label{eq:eq3}
		|\xi_k-\bin 0.\theta_{k,1}\theta_{k,2}\cdots\theta_{k,N}|\le 2^{-N}.
	\end{equation}
	
	For each $j\in \{1,2,\cdots,N\}$, by Proposition~\ref{prop:bitsExtract} (set $K=J^d$ therein), there exists $a_j\in [0,\tfrac12)$ such that
	\begin{equation*}
		\sigma_3(2^k\cdot a_j)=\theta_{k,j}, \quad \tn{for any $k\in \{1,2,\cdots,J^d\}$.}
	\end{equation*}
	Define 
	\[\phi_{3}(x)\coloneqq \sum_{j=1}^N 2^{-j}\sigma_3\big(a_j\cdot\sigma_2(x)\big)= \sum_{j=1}^N 2^{-j}\sigma_3(2^x\cdot a_j),\quad \tn{for any $x\in\R$}.\] 
	Then, for any $k\in \{1,2,\cdots,J^d\}$, we have
	\begin{equation}\label{eq:eq4}
		\phi_{3}(k)=\sum_{j=1}^N 2^{-j}\sigma_3(2^k\cdot a_j)=\sum_{j=1}^N 2^{-j}\cdot \theta_{k,j}=\bin 0.\theta_{k,1}\theta_{k,2}\cdots\theta_{k,N}.
	\end{equation}

	\mystep{5}{	
		Define $\tildephi\coloneqq \phi_3\circ\phi_2\circ\bm{\Phi}_1$ approximating $\tildef$ well, and re-scale and shift $\tildephi$ to obtain $\phi$ approximating $f$ well.}
	Define $\tildephi\coloneqq \phi_3\circ\phi_2\circ\bmPhi_1$, by Equation~\eqref{eq:eq1}, \eqref{eq:eq2}, \eqref{eq:eq3}, and \eqref{eq:eq4}, we have, for any $\bmx\in Q_\bmbeta$ and each $\bmbeta\in \{0,1,\cdots,J-1\}^d$ with $k=\phi_2(\bmbeta)$,
	\begin{equation*}
		\begin{split}
			|\tildephi(\bmx)-\tildef(\bmx)|
			&\le|\phi_3\circ\phi_2\circ\bmPhi_1(\bmx)-\tildef(\bmx_\bmbeta)|+|\tildef(\bmx_\bmbeta)-\tildef(\bmx)|\\
			&\le |\phi_3\circ\phi_2(\bmbeta)-\tildef(\bmx_\bmbeta)|+\omega_{\tildef}(\tfrac{\sqrt{d}}{J})
			\le |\phi_3(k)-\xi_k|+\omega_{\tildef}(\tfrac{\sqrt{d}}{J})\\
			&\le 
			|\bin 0.\theta_{k,1}\theta_{k,2}\cdots\theta_{k,N}-\xi_k|+\omega_{\tildef}(\tfrac{\sqrt{d}}{J})\le 2^{-N}+\omega_{\tildef}(\tfrac{\sqrt{d}}{J}).
		\end{split}
	\end{equation*}
	
	Finally, define $\phi\coloneqq 2\omega_f(\sqrt{d})\tildephi +f(\bmzero)-\omega_f(\sqrt{d})$. Equation~\eqref{eq:tildef} implies $\omega_f(r)=2\omega_f(\sqrt{d})\omega_\tildef(r)$ for any $r\ge 0$, deducing
	\begin{equation*}
		\begin{split}
			|\phi(\bmx)-f(\bmx)|=2\omega_f(\sqrt{d}) |\tildephi(\bmx)-\tildef(\bmx)|
			&	\le 2\omega_f(\sqrt{d})\big(2^{-N}+\omega_{\tildef}(\tfrac{\sqrt{d}}{J})\big)\\
			&=2\omega_f(\sqrt{d})2^{-N}+\omega_{f}(\tfrac{\sqrt{d}}{J})\\
			&
			=2\omega_f(\sqrt{d})2^{-N}+\omega_{f}(\sqrt{d}\,2^{-N}),
		\end{split}
	\end{equation*}
	for any $\bmx\in \bigcup_{\bmbeta\in\{0,1,\cdots,J-1\}^d}Q_\bmbeta=[0,1)^d$.
	It follows from $J=2^N$ and the definitions of $\bmPhi_1$, $\phi_2$, and $\phi_3$ that
	\begin{equation*}
		\begin{split}
			\phi(\bmx)
			&=2\omega_f(\sqrt{d})\phi_3\circ\phi_2\circ\bmPhi_1(\bmx) +f(\bmzero)-\omega_f(\sqrt{d})\\
			&=2\omega_f(\sqrt{d}) \phi_3\Big(\, 1+\sum_{i=1}^{d} J^{i-1}\sigma_1(J x_i)\, \Big)+f(\bmzero)-\omega_f(\sqrt{d})\\
			&=2\omega_f(\sqrt{d})\sum_{j=1}^ N 2^{-j} \sigma_3\bigg(\, a_j\cdot\sigma_2\Big(\, 1+\sum_{i=1}^{d} 2^{(i-1)N}\sigma_1(2^N x_i)\, \Big)\, \bigg)+f(\bmzero)-\omega_f(\sqrt{d}).
		\end{split}
	\end{equation*}
	So we finish the proof.
\end{proof}

\section{Approximation with continuous activation functions}\label{sec:more}
\black{
As discussed previously, our FLES networks can attain
super approximation power.
However,  two activation functions in FLES networks
are piecewise constant functions that would lead to challenges in numerical algorithm design. 
It is interesting to explore continuous activation functions achieving similar results.
To this end, we introduce three new activation functions as follows.
First, for any $\delta\in (0,1)$, we  define
\begin{equation*}
	\varrho_{1,\delta}(x)
	\coloneqq
	\left\{
	\begin{array}{ll}
		n-1,\ &x\in [n-1,n-\delta], \\
		(x-n+\delta)/\delta,\ & x\in (n-\delta,n],
	\end{array}
	\right.
	\quad  \tn{ for any }n\in \Z.
\end{equation*}
In fact, $\varrho_{1,\delta}$ can be regarded as a ``continuous version'' of the floor function.
Next, we define
\begin{equation*}
	\varrho_2(x)\coloneqq 3^x,\quad \tn{and}\quad \varrho_3(x)\coloneqq \widetilde{\calT}\big(\cos(2\pi  x)\big),\quad \tn{for any $x\in \R$,}
\end{equation*}
where 
\begin{equation*}
	\tildecalT(x)
	\coloneqq \left\{
	\begin{array}{ll}
		0,\ &  x\in (\cos(\tfrac{4\pi}{9}),\infty),\\
		1-x/\cos(\tfrac{4\pi}{9}),\ &x\in [0,\cos(\tfrac{4\pi}{9})],\\
		1,\ &  x\in (-\infty,0)
	\end{array}
	\right.
\end{equation*}
is a continuous piecewise linear function. $\varrho_2$ plays the same role of $\sigma_2(x)=2^x$ and $\varrho_3$ is essentially a ``continuous version'' of $\sigma_3$ in FLES networks.

With these three activation functions in hand, we have the following theorem.
\begin{theorem}\label{thm:main:cos}
Let $f$ be an arbitrary continuous function defined on $[0,1]^d$. For any $\delta \in (0,1)$, $N\in \N^+$, and  $p\in[1,\infty)$,  there exist $a_1,a_2,\cdots,a_ N\in[0,\tfrac29)$ such that
	\begin{equation*}
		\mathResize[0.93]{
			\|\phi-f\|_{L^p([0,1]^d)}^p\le \Big(2\omega_f(\sqrt{d})2^{-N}+\omega_{f}(\sqrt{d}\,2^{-N})\Big)^p+ 2d\delta\big(|f(\bmzero)|+\omega_f(\sqrt{d})\big)^p,
		}
	\end{equation*}
	where $\phi$ is defined by a formula in $a_1,a_2,\cdots,a_N$ as follows
	\begin{equation*}
			\mathResize[0.93]{
				\phi(\bmx)=2\omega_f(\sqrt{d})\sum_{j=1}^ N 2^{-j} \varrho_3\bigg(\, a_j\cdot\varrho_2\Big(\, 1+\sum_{i=1}^{d} 2^{(i-1)N}\varrho_{1,\delta}(2^{N} x_i)\, \Big)\, \bigg)+f(\bmzero)-\omega_f(\sqrt{d}).
				}
	\end{equation*}
\end{theorem}

The approximation error in Theorem~\ref{thm:main:cos} is characterized by $L^p$-norm for $p\in[1,\infty)$ instead of a pointwise  error estimate in  Theorem~\ref{thm:main}. By using  ideas in \cite{Shen3,shijun:thesis}, we can extend this result to $L^\infty$-norm.  However, this extension requires $2d+3$ hidden layers instead of $3$ hidden layers.
Since our focus here is the approximation using three hidden layers, we will leave this extension as future work.
}

To prove Theorem~\ref{thm:main:cos}, we need the following proposition.
\begin{proposition}
	\label{prop:bitsExtract:cos}
	Given any $K\in \N^+$ and arbitrary $\theta_1,\theta_2,\cdots,\theta_K\in \{0,1\}$, it holds that
	\begin{equation*}
		\varrho_3\big(a\cdot \varrho_2(k)\big)=\varrho_3(3^k\cdot a)=\theta_k,\quad \tn{for any $k\in\{1,2,\cdots,K\}$,}
	\end{equation*}
	where 
	\begin{equation*}
		a=\sum_{j=1}^{K} 3^{-j-1}\cdot\theta_j\ \in [0,\tfrac29).
	\end{equation*}
\end{proposition}
\begin{proof}
	
	Since $\theta_j\in\{0,1\}$ for $j\in\{1,2,\cdots,K\}$,  we have \[0\le \sum_{j=1}^{K} 3^{-j-1}\cdot\theta_j\le \sum_{j=1}^{K} 3^{-j-1}<\tfrac29,\] implying
	$a\in[0,\tfrac29)$.
	
	Next, fix $k\in\{1,2,\cdots,K\}$ for the proof below. It holds that
	\begin{equation}\label{eq:3term:cos}
		3^k\cdot  a=3^k\cdot \sum_{j=1}^{K} 3^{-j-1}\cdot\theta_j
		=\underbrace{\sum_{j=1}^{k-1}3^{k-j-1}\cdot\theta_j}_{\tn{an integer}}
		\ +\
		\overbrace{\tfrac13 \theta_k}^{\tn{$0$ or $\tfrac13$}}
		\ +\  
		\underbrace{\sum_{j=k+1}^{K}3^{k-j-1}\cdot\theta_j}_{\tn{in $[0,\tfrac29)$}}.
	\end{equation}
	Clearly, the first term $\sum_{j=1}^{k-1}3^{k-j-1}\cdot\theta_j$ in Equation~\eqref{eq:3term:cos} is a non-negative integer since $\theta_j\in\{0,1\}$ for any $j\in \{1,2,\cdots,K\}$. As for the third term in Equation~\eqref{eq:3term:cos}, we have
	\begin{equation*}
		0\le \sum_{j=k+1}^{K}3^{k-j-1}\cdot\theta_j\le \sum_{j=k+1}^{K}3^{k-j-1}<\tfrac29.
	\end{equation*}
Recall that
 	\begin{equation*}
 	\cos(2\pi  x)\in (\cos(\tfrac{4\pi}{9}),1],\quad \tn{ for any $x\in \bigcup_{n\in\N}[n,n+\tfrac29)$,}
 	\end{equation*} 
	 and 
	 \begin{equation*}
	\cos(2\pi x)\in [-1,\cos(\tfrac{2\pi}{3})]\subseteq [-1,0], \quad \tn{for any $x\in \bigcup_{n\in\N}[n+\tfrac13,n+\tfrac59)$.}
 	\end{equation*}
Note that
 \begin{equation*}
 	\tildecalT(x)
 	\coloneqq \left\{
 	\begin{array}{ll}
 		0,\ &  x\in (\cos(\tfrac{4\pi}{9}),\infty),\\
 		1-x/\cos(\tfrac{4\pi}{9}),\ &x\in [0,\cos(\tfrac{4\pi}{9})],\\
 		1,\ &  x\in (-\infty,0).
 	\end{array}
 	\right.
 \end{equation*}
	Therefore, if $\theta_k=0$, by Equation~\eqref{eq:3term:cos},  we have
	\begin{equation*}
		3^k\cdot  a\in  \bigcup_{n\in\N}[n,n+\tfrac29) 
		\quad   \Longrightarrow\quad   
		 \varrho_3(3^k\cdot  a)=\tildecalT(\cos(2\pi\cdot 3^k\cdot  a))=0=\theta_k.
	\end{equation*}
	Similarly, if $\theta_k=1$, by Equation~\eqref{eq:3term:cos}, we have 
	\begin{equation*}
	3^k\cdot  a\in  \bigcup_{n\in\N}[n+\tfrac13,n+\tfrac59) 
	\quad   \Longrightarrow\quad   
	\varrho_3(3^k\cdot  a)=\tildecalT(\cos(2\pi\cdot 3^k\cdot  a))=1=\theta_k.
\end{equation*}
	
	Since $k\in\{1,2,\cdots,K\}$ is arbitrary, we have $\varrho_3\big(a\cdot \varrho_2(k)\big)=\varrho_3(3^k\cdot  a)=\theta_k$ for any $k\in\{1,2,\cdots,K\}$. So we finish the proof.
\end{proof}

Before proving Theorem~\ref{thm:main:cos}, let us 
 define a small region as follows to simplify the notation.
Given any $J\in
\N^+$ and $\delta\in(0,1)$, define a small region $\Lambda([0,1]^d,J,\delta)$  as 
\begin{equation*}
	\Lambda([0,1]^d,J,\delta)\coloneqq\bigcup_{i=1}^{d} \bigg\{\bmx=(x_1,\cdots,x_d)\in [0,1]^d: x_i\in \bigcup_{j=1}^{J-1}[\tfrac{j-\delta}{J},\tfrac{j}{J}]\bigg\}.
\end{equation*}
In particular, $\Lambda([0,1]^d,J,\delta)=\emptyset$ if $J=1$. See Figure~\ref{fig:smallRegion} for two examples.
\begin{figure}[!htp]
	\centering
	\begin{minipage}{0.98\textwidth}
		\centering
		\begin{subfigure}[b]{0.4\textwidth}
			\centering            
			\includegraphics[width=0.999\textwidth]{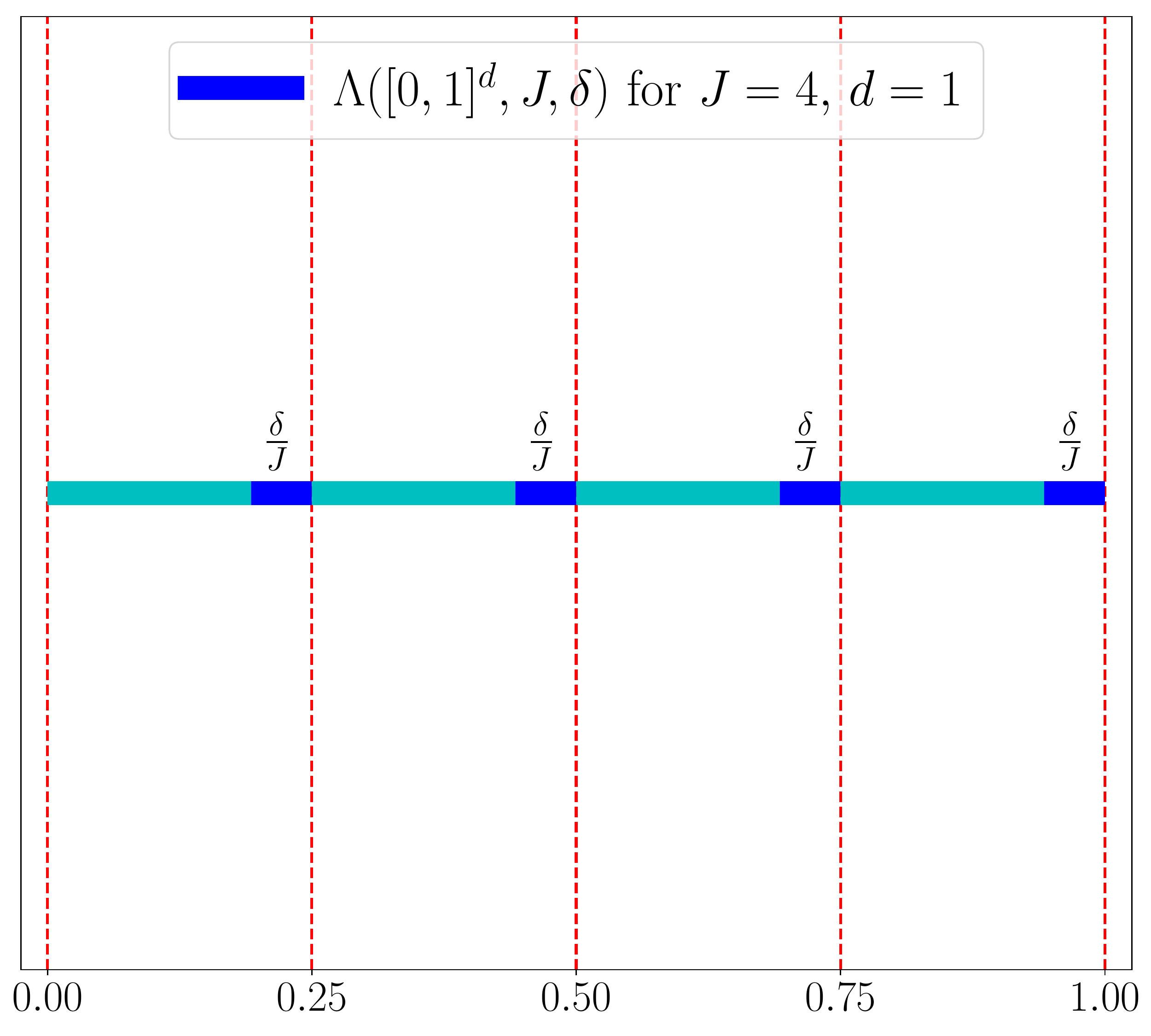}
			\subcaption{}
		\end{subfigure}
		\begin{minipage}{0.06\textwidth}
			\
		\end{minipage}
		\begin{subfigure}[b]{0.4\textwidth}
			\centering            
			\includegraphics[width=0.999\textwidth]{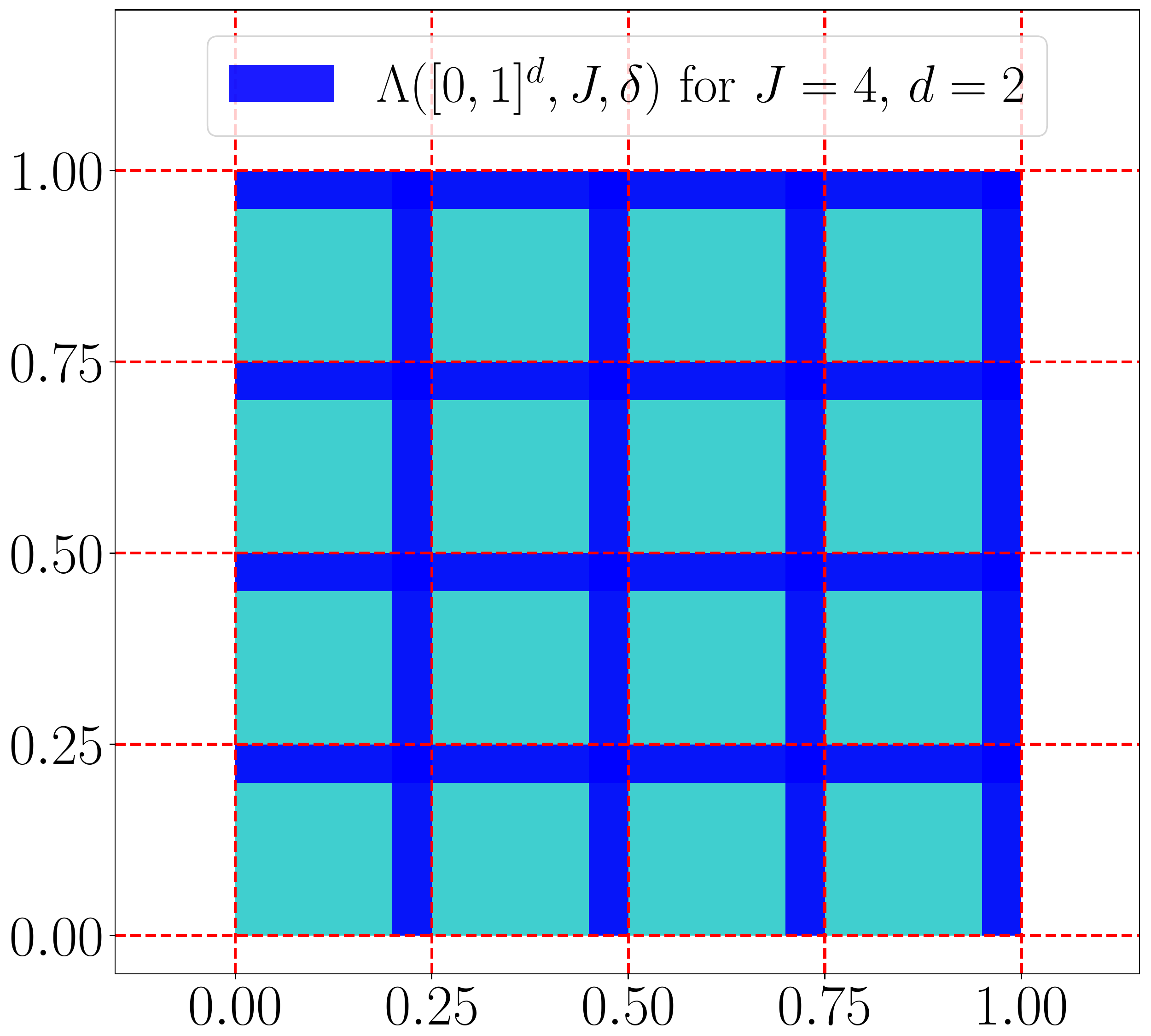}
			\subcaption{}
		\end{subfigure}
	\end{minipage}
	\caption{Illustrations of  $\Lambda([0,1]^d,J,\delta)$. (a) $J=4,\ d=1$. (b)  $J=4,\ d=2$.}
	\label{fig:smallRegion}
\end{figure}

With Proposition~\ref{prop:bitsExtract:cos} in hand, we are ready to prove Theorem~\ref{thm:main:cos}.
\begin{proof}[Proof of Theorem~\ref{thm:main:cos}]
	The proof consists of five steps.
	\mystep{1}{Set up.}
	Assume $f$ is not a constant function since it is a trivial case. Then $\omega_f(r)>0$ for any $r>0$. Clearly, $|f(\bmx)-f(\bmzero)|\le \omega_f(\sqrt{d})$ for any $\bmx\in [0,1]^d$. Define 
	\begin{equation}
		\label{eq:tildef:cos}
		\tildef\coloneqq\frac{f-f(\bmzero)+\omega_f(\sqrt{d})}{2\omega_f(\sqrt{d})}.
	\end{equation} 
	It follows that $\tildef(\bmx)\in [0,1]$ for any $\bmx\in [0,1]^d$.

	Set $J=2^N$ and divide $[0,1]^d$ into $J^d$ cubes $\{Q_\bmbeta\}_\bmbeta$ and a small region $\Lambda([0,1]^d,J,\delta)$. To be exact,  define $\bmx_\bmbeta \coloneqq \bmbeta/J$ and 
	\begin{equation*}
		Q_\bmbeta\coloneqq \Big\{\bmx=(x_1,x_2,\cdots,x_d):x_i\in [\tfrac{\beta_i}{J},\tfrac{\beta_i+1-\delta}{J}] \tn{ for } i=1,2,\cdots,d\Big\},
	\end{equation*}
	for each $\bmbeta=(\beta_1,\beta_2,\cdots,\beta_d)\in \{0,1,\cdots,J-1\}^d$.
	See Figure~\ref{fig:smallRegion+Qbeta} for illustrations.
	
	\begin{figure}[!htp]
		\centering
		\begin{minipage}{0.98\textwidth}
			\centering
			\begin{subfigure}[b]{0.4\textwidth}
				\centering            
				\includegraphics[width=0.999\textwidth]{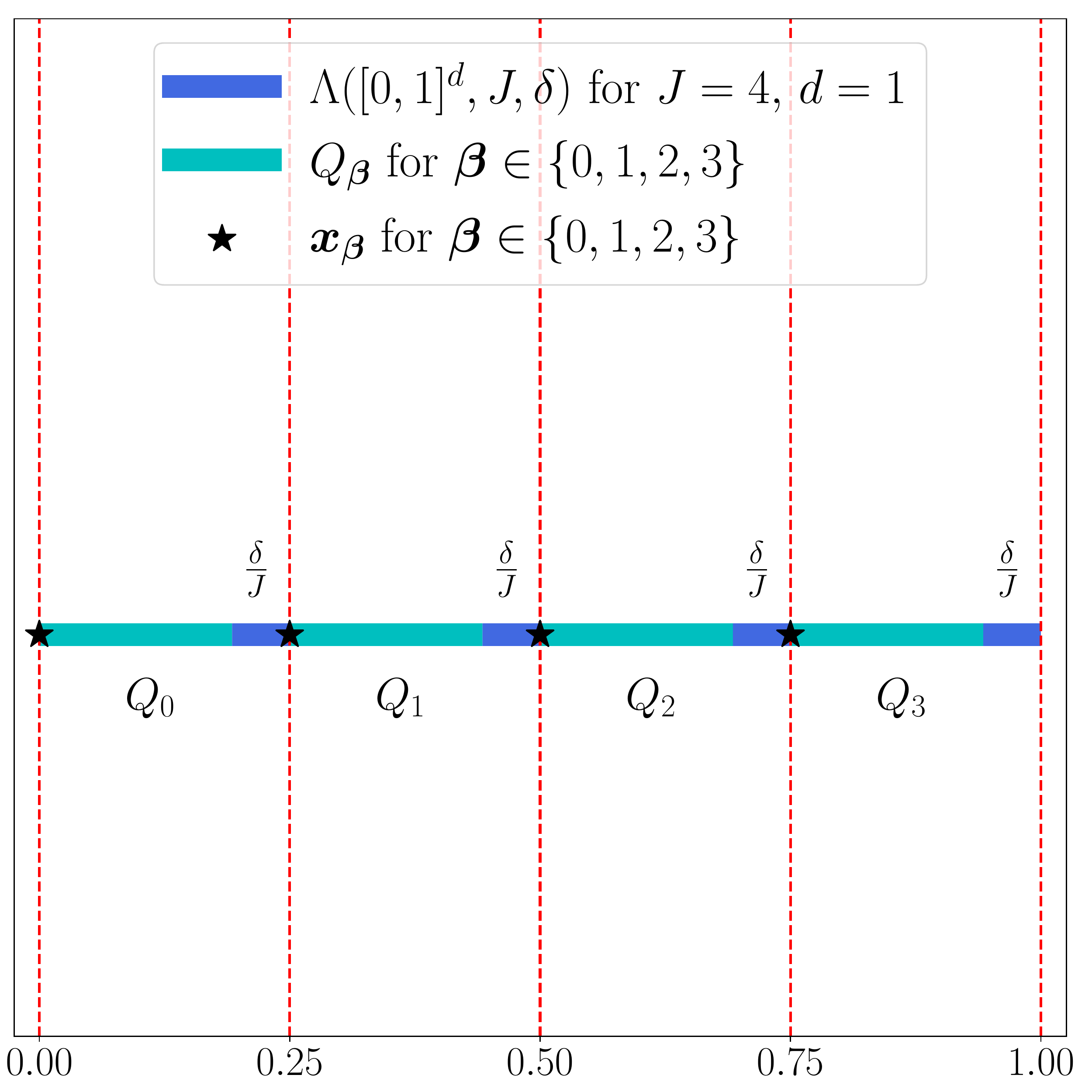}
				\subcaption{}
			\end{subfigure}
			\begin{minipage}{0.06\textwidth}
				\
			\end{minipage}
			\begin{subfigure}[b]{0.4\textwidth}
				\centering            
				\includegraphics[width=0.999\textwidth]{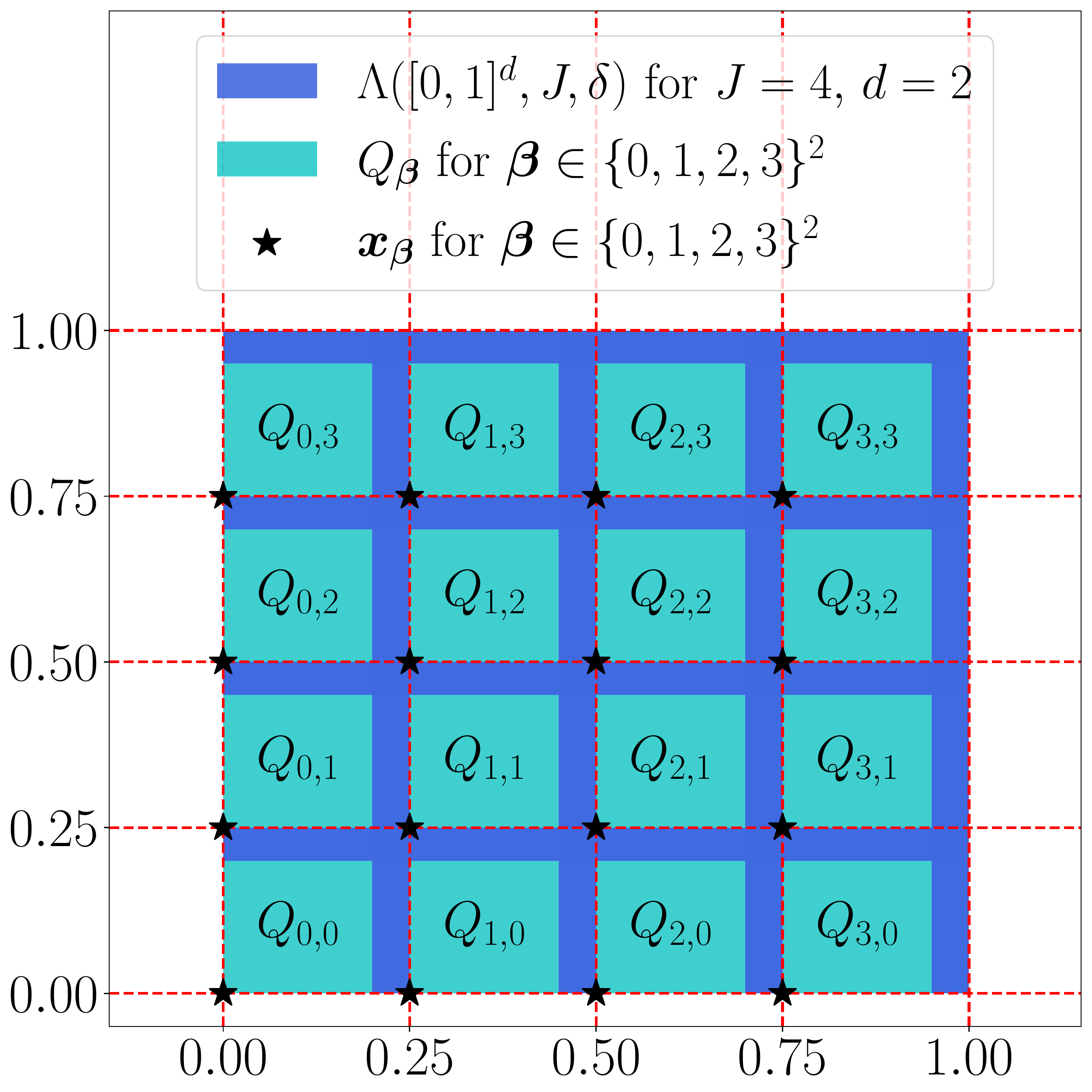}
				\subcaption{}
			\end{subfigure}
		\end{minipage}
		\caption{Illustrations of  $\Lambda([0,1]^d,J,\delta)$, $Q_\bmbeta$, and $\bmx_\bmbeta$ for any $\bmbeta\in \{0,1,\cdots,J-1\}^d$. (a) $J=4,\ d=1$. (b)  $J=4,\ d=2$.}
		\label{fig:smallRegion+Qbeta}
	\end{figure}

	\mystep{2}{Construct $\bmPhi_1$ mapping $\bmx\in Q_\bmbeta$ to $\bmbeta$ for each $\bmbeta\in\{0,1,\cdots,J-1\}^d$.}
	Define 
	\begin{equation*}
		\bmPhi_1(\bmx)\coloneqq \Big(\varrho_{1,\delta}(Jx_1), \varrho_{1,\delta}(Jx_2), \cdots ,\varrho_{1,\delta}(Jx_d) \Big),
	\end{equation*}
	\tn{for any $\bmx=(x_1,x_2,\cdots,x_d)\in\R^d$.}
	Then, for any $\bmx\in Q_\bmbeta$ and each $\bmbeta\in\{0,1,\cdots,J-1\}^d$, we have
	\begin{equation}\label{eq:eq1:cos}
		\bmPhi_1(\bmx)=\Big(\varrho_{1,\delta}(Jx_1), \varrho_{1,\delta}(Jx_2), \cdots ,\varrho_{1,\delta}(Jx_d) \Big)=(\beta_1,\beta_2,\cdots,\beta_d)=\bmbeta.
	\end{equation}

	\mystep{3}{Construct $\phi_2$ bijectively mapping $\bmbeta\in\{0,1,\cdots,J-1\}^d$ to $\phi_2(\bmbeta)\in \{1,2,\cdots,J^d\}$.}
	Inspired by the $J$-ary representation, we define an affine linear map 
	\begin{equation*}
		\phi_2(\bmx)\coloneqq 1+\sum_{i=1}^{d}J^{i-1} x_i,\quad \tn{for each $\bmx=(x_1,x_2,\cdots,x_d)\in\R^d$.}
	\end{equation*}
	Then $\phi_2$ is a bijection from $\{0,1,\cdots,J-1\}^d$ to $\{1,2,\cdots,J^d\}$.

	\mystep{4}{Construct $\phi_3$ mapping $\phi_2(\bmbeta)\in\{1,2,\cdots,J^d\}$ approximately to $\tildef(\bmx_\bmbeta)$.}
	For each $k\in\{1,2,\cdots,J^d\}$, there exists a unique $\bmbeta\in\{0,1,\cdots,J-1\}^d$ such that $\phi_2(\bmbeta)=k$. Thus, define
	\begin{equation}\label{eq:eq2:cos}
		\xi_{k}\coloneqq \tildef(\bmx_\bmbeta)\in [0,1],\quad \tn{for any $k\in\{1,2,\cdots,J^d\}$ with $k=\phi_2(\bmbeta)$.}
	\end{equation}
	For each $k\in \{1,2,\cdots,J^d\}$, there exist $\theta_{k,1},\theta_{k,2},\cdots,\theta_{k,N}\in \{0,1\}$ such that
	\begin{equation}\label{eq:eq3:cos}
		|\xi_k-\bin 0.\theta_{k,1}\theta_{k,2}\cdots\theta_{k,N}|\le 2^{-N}.
	\end{equation}
	
	For each $j\in \{1,2,\cdots,N\}$, by Proposition~\ref{prop:bitsExtract:cos} (set $K=J^d$ therein), there exists $a_j\in [0,\tfrac29)$ such that
	\begin{equation*}
		\varrho_3(3^k\cdot a_j)=\theta_{k,j}, \quad \tn{for any $k\in \{1,2,\cdots,J^d\}$.}
	\end{equation*}
	Define 
	\[\phi_{3}(x)\coloneqq \sum_{j=1}^N 2^{-j}\varrho_3\big(a_j\cdot\varrho_2(x)\big)= \sum_{j=1}^N 2^{-j}\varrho_3(3^x\cdot a_j),\quad \tn{for any $x\in\R$}.\] 
	Then we have
	\begin{equation}\label{eq:phiUB}
		\varrho_3(x)\in [0,1],\quad \tn{for any $x\in\R$} \quad \Longrightarrow\quad \phi_{3}(x)\in [0,1],\quad \tn{for any $x\in \R$,}
	\end{equation}
	and 
	\begin{equation}\label{eq:eq4:cos}
		\phi_{3}(k)=\sum_{j=1}^N 2^{-j}\varrho_3(3^k\cdot a_j)=\sum_{j=1}^N 2^{-j}\cdot \theta_{k,j}=\bin 0.\theta_{k,1}\theta_{k,2}\cdots\theta_{k,N},
	\end{equation}
	for any $k\in \{1,2,\cdots,J^d\}$.

	\mystep{5}{	
		Define $\tildephi\coloneqq \phi_3\circ\phi_2\circ\bm{\Phi}_1$ approximating $\tildef$ well, and re-scale and shift $\tildephi$ to obtain $\phi$ approximating $f$ well.}
	Define $\tildephi\coloneqq \phi_3\circ\phi_2\circ\bmPhi_1$, by Equation~\eqref{eq:eq1:cos}, \eqref{eq:eq2:cos}, \eqref{eq:eq3:cos}, and \eqref{eq:eq4:cos}, we have, for any $\bmx\in Q_\bmbeta$ and each $\bmbeta\in \{0,1,\cdots,J-1\}^d$ with $k=\phi_2(\bmbeta)$,
	\begin{equation*}
		\begin{split}
			|\tildephi(\bmx)-\tildef(\bmx)|
			&\le |\phi_3\circ\phi_2\circ\bmPhi_1(\bmx)-\tildef(\bmx_\bmbeta)|+|\tildef(\bmx_\bmbeta)-\tildef(\bmx)|\\
			&\le |\phi_3\circ\phi_2(\bmbeta)-\tildef(\bmx_\bmbeta)|+\omega_{\tildef}(\tfrac{\sqrt{d}}{J})
			\le |\phi_3(k)-\xi_k|+\omega_{\tildef}(\tfrac{\sqrt{d}}{J})\\
			&\le 
			|\bin 0.\theta_{k,1}\theta_{k,2}\cdots\theta_{k,N}-\xi_k|+\omega_{\tildef}(\tfrac{\sqrt{d}}{J})\le 2^{-N}+\omega_{\tildef}(\tfrac{\sqrt{d}}{J}).
		\end{split}
	\end{equation*}
	
	Finally, define $\phi\coloneqq 2\omega_f(\sqrt{d})\tildephi +f(\bmzero)-\omega_f(\sqrt{d})$. Equation~\eqref{eq:tildef:cos} implies $\omega_f(r)=2\omega_f(\sqrt{d})\omega_\tildef(r)$ for any $r\ge 0$, deducing
	\begin{equation*}
		\begin{split}
			|\phi(\bmx)-f(\bmx)|=2\omega_f(\sqrt{d}) |\tildephi(\bmx)-\tildef(\bmx)|
			&	\le 2\omega_f(\sqrt{d})\big(2^{-N}+\omega_{\tildef}(\tfrac{\sqrt{d}}{J})\big)\\
			&=2\omega_f(\sqrt{d})2^{-N}+\omega_{f}(\tfrac{\sqrt{d}}{J}),\\
		\end{split}
	\end{equation*}
	for any $\bmx\in \bigcup_{\bmbeta\in\{0,1,\cdots,J-1\}^d}Q_\bmbeta$. 
	By Equation~\eqref{eq:phiUB} and the definition of 
	\begin{equation*}
		\phi= 2\omega_f(\sqrt{d})\phi_3\circ\phi_2\circ\bm{\Phi}_1 +f(\bmzero)-\omega_f(\sqrt{d}),
	\end{equation*}
we have $\|\phi\|_{L^\infty(\R^d)}\le |f(\bmzero)|+\omega_f(\sqrt{d})$.
	Let $\mu(\cdot)$ denote the Lebesgue measure. Note that $\|f\|_{L^\infty([0,1]^d)}\le |f(\bmzero)|+\omega_f(\sqrt{d})$.
	If follows from $\mu(\Lambda([0,1]^d, J,\delta))\le Jd \tfrac{\delta}{J}=d\delta$ that 	
	\begin{equation*}
		\begin{split}
		&\hspace{16pt}\|\phi-f\|_{L^p([0,1]^d)}^p=\int_{[0,1]^d} |\phi(\bmx)-f(\bmx)|^p d\bmx\\
		&=\sum_{\bmbeta\in \{0,1,\cdots,J-1\}^d} \int_{Q_\bmbeta} |\phi(\bmx)-f(\bmx)|^p d\bmx
				+ \int_{\Lambda([0,1]^d,J,\delta)} |\phi(\bmx)-f(\bmx)|^p d\bmx\\
		&\le  \sum_{\bmbeta\in \{0,1,\cdots,J-1\}^d} \mu(Q_\bmbeta)\Big(2\omega_f(\sqrt{d})2^{-N}+\omega_{f}(\tfrac{\sqrt{d}}{J})\Big)^p
		+  \big(2|f(\bmzero)|+2\omega_f(\sqrt{d})\big)^p  d\delta\\
		&\le \Big(2\omega_f(\sqrt{d})2^{-N}+\omega_{f}(\sqrt{d}\,2^{-N})\Big)^p
		+ 2^p d\delta\big(|f(\bmzero)|+\omega_f(\sqrt{d})\big)^p.
		\end{split}
	\end{equation*}

	By the definitions of $\bmPhi_1$, $\phi_2$, and $\phi_3$, we have
	\begin{equation*}
		\begin{split}
			\phi(\bmx)
			&=2\omega_f(\sqrt{d})\phi_3\circ\phi_2\circ\bmPhi_1(\bmx) +f(\bmzero)-\omega_f(\sqrt{d})\\
			&=2\omega_f(\sqrt{d}) \phi_3\Big(\, 1+\sum_{i=1}^{d} J^{i-1}\varrho_{1,\delta}(J x_i)\, \Big)+f(\bmzero)-\omega_f(\sqrt{d})\\
			&=2\omega_f(\sqrt{d})\sum_{j=1}^ N 2^{-j} \varrho_3\bigg(\, a_j\cdot\varrho_2\Big(\, 1+\sum_{i=1}^{d} 2^{(i-1)N}\varrho_{1,\delta}(2^N x_i)\, \Big)\, \bigg)+f(\bmzero)-\omega_f(\sqrt{d}).
		\end{split}
	\end{equation*}
	So we finish the proof.
\end{proof}

\section{Conclusion}
\label{sec:conclusion}

This paper has introduced a theoretical framework to show that three hidden layers are enough for neural network approximation to achieve exponential convergence and avoid the curse of dimensionality for approximating functions as general as (H\"older) continuous functions. The key idea is to leverage the power of multiple simple activation functions: the floor function ($\lfloor x\rfloor$), the exponential function ($2^x$), the step function ($\one_{x\geq 0}$), or their compositions. This new class of networks is called the FLES network. Given a Lipschitz continuous function $f$ on $[0,1]^d$, it was shown by construction that FLES networks with width $\max\{d,\, N\}$ and three hidden layers admit a uniform approximation rate $6\lambda\sqrt{d}\,2^{-N}$, where $\lambda$ is the Lipschitz constant of $f$. More generally for an arbitrary continuous function $f$ on $[0,1]^d$ with a modulus of continuity $\omega_f(\cdot)$, the constructive approximation rate is $2\omega_f(2\sqrt{d}){2^{-N}}+ \omega_f(2\sqrt{d}\,2^{-N})$. \black{We also extend such a result to general bounded continuous functions on a bounded set $E\subseteq \R^d$.}
The results in this paper provide a theoretical lower bound of the power of FLES networks. Whether or not this bound is achievable in actual computation relies on advanced algorithm design as a separate line of research. \black{Finally, we have also derived similar approximation results in the $L^p$-norm for $p\in[1,\infty)$ using continuous activation functions.}

\vspace{0.5cm}

{\bf Acknowledgments.} Z.~Shen is supported by Tan Chin Tuan Centennial Professorship.  H.~Yang was partially supported by the US National Science Foundation under award DMS-1945029.

\bibliographystyle{elsarticle}%
\bibliography{references}%

\end{document}

%% file: preamble.tex
\usepackage{mathrsfs,amsmath,amsfonts,amssymb}
\usepackage{xcolor}
\usepackage{dsfont}
\usepackage{bm}
\usepackage{MnSymbol}
\usepackage{breakurl}
\usepackage{indentfirst}
\usepackage{enumerate} 
\usepackage{tabularx,multirow,array,booktabs}
\usepackage{colortbl}
\usepackage{mathtools,tikz}
\usepackage{graphicx}
\graphicspath{{figures/}}%
\usepackage{float}
\usepackage{subcaption} 
\usepackage{amsthm}
\theoremstyle{plain}
\newtheorem{theorem}{Theorem}[section]%
\newtheorem{corollary}[theorem]{Corollary} %

\newtheorem{proposition}[theorem]{Proposition}
\theoremstyle{definition}
\newtheorem{definition}[theorem]{Definition}
\theoremstyle{remark}

\newcommand{\R}{\ifmmode\mathbb{R}\else$\mathbb{R}$\fi}
\newcommand{\N}{\ifmmode\mathbb{N}\else$\mathbb{N}$\fi}
\newcommand{\Z}{\ifmmode\mathbb{Z}\else$\mathbb{Z}$\fi}
\newcommand{\Q}{\ifmmode\mathbb{Q}\else$\mathbb{Q}$\fi}
\let\tn\textnormal
\newcommand{\bmx}{{\bm{x}}}

\newcommand{\bmy}{{\bm{y}}}

\newcommand{\bmzero}{{\bm{0}}}

\newcommand{\bmtheta}{{\bm{\theta}}}

\newcommand{\bmbeta}{{\bm{\beta}}}

\newcommand{\bmPhi}{{\bm{\Phi}}}

\newcommand{\calE}{{\mathcal{E}}}

\newcommand{\calN}{{\mathcal{N}}}
\newcommand{\calO}{{\mathcal{O}}}

\newcommand{\calD}{{\mathcal{D}}}

\newcommand{\calS}{{\mathcal{S}}}
\newcommand{\calT}{{\mathcal{T}}}
\newcommand{\tildephi}{{\widetilde{\phi}}}

\newcommand{\tildeg}{{\widetilde{g}}}

\newcommand{\tildef}{{\widetilde{f}}}

\newcommand{\tildecalT}{{\widetilde{\calT}}}

\newcommand{\scrF}{{\mathscr{F}}}

\newcommand{\vcd}{\tn{VCDim}}
\newcommand{\holder}{\tn{H\"older}}

\DeclareMathOperator*{\argmin}{arg\,min}
\usepackage{anyfontsize}         
\def\one{{\ensuremath{\mathds{1}}}}
\newcommand{\bin}{\tn{bin}\hspace{1.2pt}}
\newcommand{\mystep}[2]{\par \vspace{0.25cm}\noindent\textbf{\hspace{8pt}Step }$#1\colon$ #2 \vspace{0.18cm} \par }

\newcommand{\setMathResizeRate}[1]{\def\mathResizeRate{#1}}
\setMathResizeRate{0.8}
\newcommand{\mathResize}[2][\mathResizeRate]{
	\scalebox{#1}[#1]{\(\displaystyle #2\)}
}
\usepackage[bookmarks,colorlinks]{hyperref}
\definecolor{myturquoise}{HTML}{008080}
\hypersetup{citecolor=myturquoise,linkcolor=blue,urlcolor=red}
\usepackage[left,mathlines]{lineno}
\definecolor{mycolor}{HTML}{BEBEBE}

\makeatletter
\newcommand*\patchAmsMathEnvironmentForLineno[1]{%
	\expandafter\let\csname old#1\expandafter\endcsname\csname #1\endcsname
	\expandafter\let\csname oldend#1\expandafter\endcsname\csname end#1\endcsname
	\renewenvironment{#1}%
	{\linenomath\csname old#1\endcsname}%
	{\csname oldend#1\endcsname\endlinenomath}}%
\newcommand*\patchBothAmsMathEnvironmentsForLineno[1]{%
	\patchAmsMathEnvironmentForLineno{#1}%
	\patchAmsMathEnvironmentForLineno{#1*}}%
\patchBothAmsMathEnvironmentsForLineno{equation}%
\patchBothAmsMathEnvironmentsForLineno{align}%
\patchBothAmsMathEnvironmentsForLineno{flalign}%
\patchBothAmsMathEnvironmentsForLineno{alignat}%
\patchBothAmsMathEnvironmentsForLineno{gather}%
\patchBothAmsMathEnvironmentsForLineno{multline}%
\makeatother